\RequirePackage[svgnames]{xcolor}

\providecommand{\paperaccent}{utOrange}

\documentclass[11pt,letterpaper]{mystyle}

\usepackage[all]{hypcap}
\usepackage[svgnames]{xcolor}
\usepackage[comma,authoryear,compress]{natbib}
\hypersetup{
    colorlinks = true,
    allcolors = {PaperAccent},
}

\usepackage{algorithm}
\usepackage{algorithmicx}
\usepackage{algpseudocode}
\usepackage{microtype}
\usepackage{graphicx}
\expandafter\def\csname ver@subfig.sty\endcsname{}
\usepackage{booktabs} %
\usepackage{arydshln}
\usepackage{float}
\usepackage{bigstrut}

\usepackage{amsmath}
\usepackage{amssymb}
\usepackage{mathtools}
\usepackage{amsthm}
\usepackage{mathrsfs}
\usepackage{nicefrac}
\usepackage{dsfont}
\usepackage{enumitem}
\usepackage{subcaption}
\usepackage{graphicx,subfig}
\usepackage{cleveref}
\usepackage{bxcoloremoji}
\usepackage{multirow}
\usepackage{float}

\usepackage[utf8]{inputenc} %
\usepackage[T1]{fontenc}    %
\usepackage{hyperref}       %
\usepackage{url}            %
\usepackage{booktabs}       %
\usepackage{amsfonts}       %
\usepackage{nicefrac}       %
\usepackage{microtype}      %
\usepackage{graphicx}
\usepackage{subcaption} %
\usepackage{amssymb}
\usepackage{fdsymbol}
\usepackage{wrapfig}
\usepackage{lipsum}
\usepackage{enumitem}
\usepackage{stackengine}
\usepackage[font=small,labelfont=bf]{caption}
\usepackage{color}
\usepackage{adjustbox}

\usepackage{rotating}
\usepackage{makecell}

\newtheorem{theorem}{Theorem}[section]
\newtheorem{corollary}[theorem]{Corollary}
\newtheorem{proposition}[theorem]{Proposition}
\newtheorem{assumption}[theorem]{Assumption}
\renewcommand{\mathbf}{\boldsymbol}
\providecommand{\xv}{\mathbf{x}}
\providecommand{\yv}{\mathbf{y}}
\providecommand{\cv}{\mathbf{c}}
\providecommand{\hv}{\mathbf{h}}
\providecommand{\qv}{\mathbf{q}}
\providecommand{\kv}{\mathbf{k}}
\providecommand{\sv}{\mathbf{s}}

\usepackage{amsmath,amsfonts,bm}
\providecommand{\eqref}[1]{equation~\ref{#1}}

\DeclareSymbolFont{gradientsymbols}{OMS}{zplm}{m}{n}
\SetSymbolFont{gradientsymbols}{bold}{OMS}{zplm}{b}{n}
\DeclareMathSymbol{\nabla}{\mathord}{gradientsymbols}{"72}

\usepackage{xspace}
\newcommand{\name}{{FOVEATED}\xspace}

\usepackage{amsmath}

\usepackage[all]{hypcap}

\title{Improving Atomic-Fact Recall via Focused Views in Unstructured Knowledge Editing}

\runningtitle{Improving Atomic-Fact Recall via Focused Views in Unstructured Knowledge Editing}

\author{
Ding Wu\textsuperscript{1,*} \quad
Ye Zhang\textsuperscript{2,*} \quad
Haoyu Wang\textsuperscript{3} \quad
Tianci Liu\textsuperscript{4} 
\\
\textsuperscript{1}Georgia Institute of Technology \quad
\textsuperscript{2}University of Pennsylvania \\
\textsuperscript{3}University at Albany \quad
\textsuperscript{4}University of Tennessee 
}

\begin{document}

\begin{abstract}
Large language models (LLMs) increasingly serve as general-purpose interfaces to factual knowledge, but their parameters do not automatically reflect information that changes after pretraining.
Knowledge editing (KE) provides a targeted alternative to costly retraining by modifying selected knowledge and preserving unrelated knowledge and general capabilities.
Conventional KE uses structured factual triples, whereas unstructured KE (UKE) uses free-form passages containing multiple facts.
Nonetheless, existing UKE editors exhibit a failure mode known as \emph{context reliance}: edited LLMs can often reproduce the editing passage but fail to reliably recall its individual facts without the original passage context.
We identify \emph{context-induced difficulty underestimation} under the standard passage-level editing objective: later facts receive increasingly rich ground-truth context and consequently incur lower initial losses, making them appear easier to learn.
In response, we propose \name, a plug-and-play framework that constructs focused views of each sentence by randomly shifting the Rotary Position Embedding (RoPE) positions assigned to the keys of its preceding context.
The perturbation is applied during editing and removed afterward, leaving the model's native positional encoding unchanged at inference time.
We instantiate \name for both direct-optimization and locate-then-edit editors.
We theoretically analyze how \name counteracts context-induced difficulty underestimation and empirically demonstrate consistent improvements across five KE editors, two LLM backbones, and three benchmarks.

\vspace{2mm}


\vspace{5mm}

\coloremojicode{1F4C5} 
\textbf{Date}: {\monthyeardate\today}

\coloremojicode{1F4E7} 
\textbf{Contact}: 
\href{mailto:tliu43@tennessee.edu}{tliu43@tennessee.edu}

\end{abstract}

\maketitle
\def\thefootnote{$^{*}$}\footnotetext{These authors contributed equally to this work.}\def\thefootnote{\arabic{footnote}}
\vspace{3mm}
\section{Introduction}
\label{sec:intro}

Large language models (LLMs) have become a foundation for modern natural
language processing tasks~\citep{vaswani2017attention,brown2020language,zhou2023comprehensive}.
As model scale continues to grow, these capabilities make LLMs increasingly
useful as general-purpose interfaces for accessing and reasoning with
knowledge~\citep{bommasani2021opportunities,roberts2020much,
kojima2022large}.

However, an LLM's parametric knowledge remains largely fixed after
pretraining and can quickly become outdated as real-world knowledge evolves,
leading to factual errors in downstream applications
~\citep{de2021editing,hartvigsen2024aging}.
Keeping the model current through repeated retraining is prohibitively
expensive.
As a remedy, 
\emph{knowledge editing} (KE) addresses this problem by incorporating specific new
knowledge through selective model updates and preserving unrelated
knowledge and general capabilities~\citep{wang2023knowledge,yao2023editing,zhang2024comprehensive}.

Standard KE typically represents new knowledge as a separate factual association, often a \textit{(subject, relation, object)} triple.
Given such structured facts, some methods directly optimize selected parameters~\citep{wang2024roselora,wang2024wise}, while
\emph{locate-then-edit} (LTE) methods first identify where the target knowledge is represented and then write a targeted update into the corresponding layers~\citep{meng2022rome,meng2023memit,
fang2025alphaedit}.
Recent work relaxes this structured-input assumption through
\emph{unstructured knowledge editing} (UKE), where the new knowledge is
provided as a free-form passage~\citep{wu2024akew,yang2025mirage}.
Unlike a structured edit, such a passage may express multiple interrelated
facts without specifying which ones a user may later query.
UKE methods accordingly expand the editable modules and decompose the passage into
smaller units and coordinate their updates~\citep{deng2024unke,jiang2025anyedit,
su2025muke}.
Together, these methods expand the scope of KE from individual factual
associations to free-form passages.

Nonetheless, existing methods often fall short of making effective use of
the edited knowledge beyond the original editing context.
In particular, an edited model may recall the complete passage under the
original editing query, yet fail to retrieve the new knowledge when answering
more targeted questions about individual facts
~\citep{zhou2026context,wang2026fable}.
Recent studies have begun to examine this failure mode, termed
\emph{context reliance}, in which the new knowledge may remain accessible
only as part of the complete passage, where the edited model can rely on the
preceding context as a shortcut during generation.
To mitigate this problem, \citet{zhou2026context} align predictions under global and
local contexts; \citet{wang2026fable} synthesize atomic question--answer
supervision with external models; and \citet{liu2026hpse} construct hybrid
self-distillation trajectories from a privileged in-context state.
\emph{Yet the learning dynamics underlying context reliance remain
underexplored.}

In light of this gap, we examine how context affects the apparent learning difficulty under the standard passage-level objective in this work. Interestingly, we find that
the pre-update editing loss decreases for facts appearing later in the passage,
even though all edited facts are new to the
model~\citep{wu2024akew,deng2024unke}.
Intuitively, under teacher forcing~\citep{lamb2016professor}, later content is predicted from
longer ground-truth prefixes and thus receives richer contextual support.
This growing support acts as a \emph{shortcut}, lowering the initial losses of later facts and
making them appear easier to learn.
We call this \emph{context-induced difficulty underestimation};
Sec.~\ref{sec:problem} details our analysis.
\emph{This observation is the first contribution of this work.}

In response, we propose \name to construct \emph{focused views} that make the
preceding context less reliable as a shortcut during editing.
By varying how a target sentence retrieves its preceding context while preserving the sentence itself, \name trains the model on each fact across randomized contextual views and directs the editing signal toward the target content.
Concretely, \name constructs such views by randomly shifting the \emph{key-side}
Rotary Position Embedding~\citep[RoPE;][]{su2021roformer} positions assigned to
the prompt and preceding sentences, while leaving all query rotations and the target sentence's key rotations unchanged.
The perturbation is applied only during editing and removed afterward, leaving the model's native positional encoding unchanged at inference time.
We further provide theoretical support for the focused-view design.

\name is a flexible framework that intervenes in the editing computation without prescribing how the updates are parameterized. It can be applied to a wide range of existing editors, and we instantiate \name for two broad families.
Direct optimization (DO) editors update selected parameters with gradient descent~\citep{hu2022lora,wang2024roselora,wang2024wise,liu2025unlocking}. 
For this family, \name generalizes the standard passage-level objective to a weighted expectation of sentence-level losses under focused views. Setting the perturbation magnitude to zero exactly recovers the standard passage-level objective.
For LTE editors~\citep{deng2024unke,su2025muke}, 
applying \name is less direct because their
passage objectives optimize intermediate hidden targets rather than the
eventual parameter updates.
We follow \citet{wang2026fable} and adopt a two-stage design by first writing sentence-level targets into shallow layers and then applying the
base editor in deeper layers to learn the complete passage.
The key distinction of our approach is that \name obtains its sentence-level
targets directly from the editing passage under focused views, without
synthesizing atomic question--answer supervision through external models.
See Sec.~\ref{sec:method} for details. 
\emph{This general editing framework and our theoretical analysis form the
main technical contribution of this work.}

The remainder of this paper is organized as follows. Sec.~\ref{sec:problem}
introduces context-induced difficulty underestimation, Sec.~\ref{sec:method}
presents \name and its theoretical analysis, and Sec.~\ref{sec:experiment}
evaluates the method. Secs.~\ref{sec:related_work} and~\ref{sec:conclusion}
review related work and conclude the paper, respectively.

\section{Difficulty Underestimation in Unstructured Knowledge Editing}
\label{sec:problem}

\begin{figure}[t]
\centering
\includegraphics[width=0.95\textwidth]{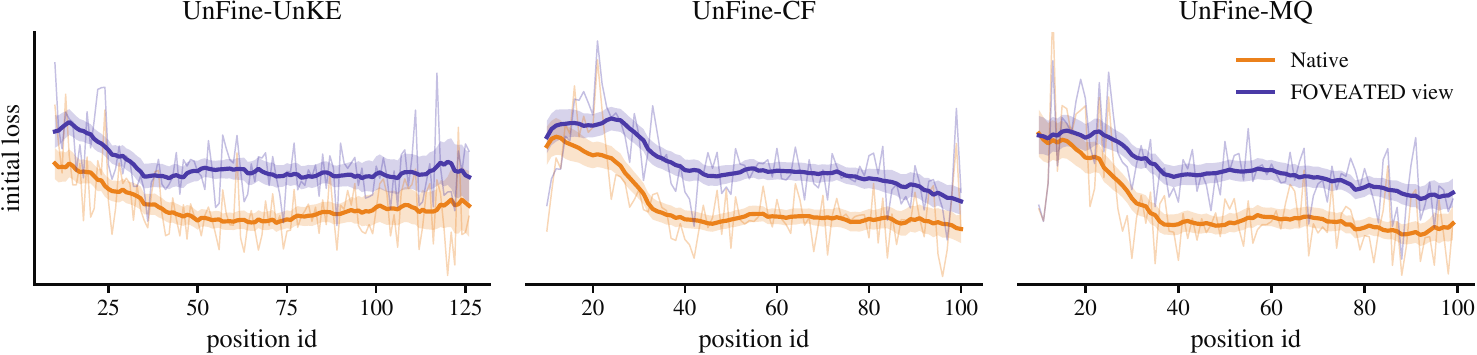}
\caption{Pre-edit NLL loss in UnFine passages. Native loss gradually decreases toward later passage positions, and \name views yield consistently higher losses at the same positions.
}
\label{fig:token-loss}
\end{figure}

This section reveals context-induced difficulty underestimation in UKE. Background is also provided.

\subsection{Preliminaries}
\label{sec:prelim}

For a text sequence $\xv$ of length $|\xv|$, an autoregressive LLM
$\pi(\cdot)$ factorizes its likelihood as
\begin{align*}
    \pi(\xv)
    &= \prod_{t=1}^{|\xv|} \pi(x_t\mid x_1,\ldots,x_{t-1})
    \triangleq \prod_{t=1}^{|\xv|} \pi(x_t\mid\xv_{<t}),
\end{align*}
where $\pi(x_t\mid\xv_{<t})$ is the predicted probability assigned to
token $x_t$ given its prefix $\xv_{<t}$.
To generate text, the LLM sequentially draws $x_t$ from
$\pi(\cdot\mid\xv_{<t})$ and appends it to the prefix,
until an end-of-sequence token is produced, or the length limit is reached~\citep{bengio2000neural}.

\textit{Knowledge Editing (KE)} modifies selected knowledge in a
pre-trained LLM and preserves knowledge outside the scope of the
edit~\citep{wang2023knowledge,zhang2024comprehensive}.
Structured KE represents the knowledge to be updated as a factual
triple \emph{(subject, relation, object)}, with the edit replacing
the object with a new value.
The edit is typically specified by a prompt-response pair
$(\xv,\yv)$, where $\xv$ identifies the subject and relation, and $\yv$ the updated object.
\textit{Unstructured KE (UKE)} instead associates $\xv$ with 
a free-form passage $\cv$ rather than explicitly annotated
triples~\citep{wu2024akew,yang2025mirage}.
A passage may express multiple facts, each a separately queryable
\emph{atomic fact}~\citep{wang2026fable}.
We focus on atomic-fact recall in UKE.

\subsection{Learning Difficulty under Teacher Forcing}
\label{sec:difficulty-underestimation}

UKE aims to edit multiple new facts encoded in a free-form passage.
However, as revealed by \citet{zhou2026context,wang2026fable},
many UKE methods encounter a \emph{context reliance} failure
mode: recall of the edited knowledge
depends heavily on the context supplied during editing.
When that context is absent, the edited
LLMs struggle to recall the new knowledge, especially
facts appearing later in the passage.
Such failure limits the effectiveness of UKE beyond
the original editing context.

Whereas this pattern appears to be
a contextual ``overfitting'' problem of KE~\citep{zhang2024uncovering,liu2025heto},
we show that it is analogous to ``underfitting'':
the presence of
contextual support makes new knowledge appear easy
to predict even when the LLM has not sufficiently
learned it.

Specifically, although (U)KE methods
differ in the variables they optimize to incorporate
new knowledge~\citep{fang2025alphaedit,jiang2025anyedit},
their supervision often comes from fitting the editing
passage through a negative log-likelihood (NLL)
loss~\citep{liu2025heto}.
Let $\theta$ denote the editor-specific
optimization variables and $\pi_\theta$ the corresponding
edited model.
Given prompt $\xv$ and passage $\cv$,
a common full-passage objective is
\begin{align}
    \mathcal{L}_{\mathrm{NLL}}(\theta)
    = -\frac{1}{|\cv|}\sum_{t=1}^{|\cv|}
    \log \pi_\theta(c_t\mid\xv,\cv_{<t}),
    \label{eq:nll}
\end{align}
where the ground-truth $\cv_{<t}$ is provided
to predict token $c_t$,
a procedure called \emph{teacher forcing}~\citep{lamb2016professor}.
Intuitively, the prefix provides
increasingly rich context as the
passage unfolds.
Consequently, even if the LLM does not yet hold the new knowledge, it can still see the fact as a plausible continuation
of the surrounding text.
This contextual shortcut can make
the loss understate the difficulty of acquiring the
knowledge.
We call this effect 
\emph{context-induced difficulty underestimation}.

To examine the effect empirically, we evaluate
Qwen2.5-7B-Instruct~\citep{qwen2024qwen25} on free-form
passages from UnFine~\citep{wang2026fable} and measure
the initial token-level NLL before any editing updates.
Although these passages are designed to contain new
knowledge throughout~\citep{deng2024unke,wang2026fable},
Fig.~\ref{fig:token-loss} shows a clear decrease in
initial loss toward later positions.
Later content thus appears easier to predict despite
still introducing new knowledge, consistent with the
growing prefix providing contextual support before
the model has learned anything from the edit.
To test whether intrinsic sentence difficulty explains the decline, we randomly shuffle the sentence order in each passage and observe the same pattern (App.~\ref{app:shuffle-control}).

To address this issue, we introduce the plug-and-play method \name in the next section.

\section{Proposed Method}
\label{sec:method}

We propose \name (\underline{FO}cused
\underline{V}i\underline{E}ws of \underline{AT}omic facts in unstructured
knowledge \underline{ED}iting) and integrate it with multiple UKE editors.

\subsection{Breaking Context Shortcuts through Positional Encoding Perturbation}
\label{sec:method:krp}

Given the difficulty underestimation induced by contextual shortcuts, we propose to make the context less reliable as a shortcut during the editing phase,
so that the fact needs to be learned consistently across randomized views of the surrounding context.

Inspired by how LLMs use positional
encoding to understand context~\citep{vaswani2017attention,press2021train}, and how
manipulating positional encoding during
fine-tuning can help them handle contexts longer than
those seen during training~\citep{zhu2024pose,peng2024yarn},
\name randomly perturbs selected
positional encodings
to break the context shortcut in a fine-grained way.
Intuitively, when positional information
is perturbed, the shortcut becomes less reliable,
encouraging the LLM to learn
the new knowledge through a more focused view of the
current content.

We focus on Rotary Position Embedding~\citep[RoPE;][]{su2021roformer},
which is widely used in modern LLMs~\citep{dubey2024llama3,qwen2024qwen25,liu2024deepseek}.
Given an attention head of even dimension $d_h$,
let $\hv_m$ denote the input hidden state at position $m$.
The attention head maps $\hv_m$ into
$d_h$-dimensional query and key vectors via
$\mathbf W_Q$ and $\mathbf W_K$, respectively.
To incorporate relative-position information into attention
score computation, RoPE rotates the query and key vectors based on their positions before
computing the attention score.
The query and key vectors lie in
$\mathbb C^{d_h/2}$, with coordinate $j$ rotated
by $m\omega_j$ at position $m$.
Let $\boldsymbol\Theta=\operatorname{diag}(\omega_1,\ldots,\omega_{d_h/2})$ collect the rotation frequencies. The rotated query and key vectors are given by
\begin{align*}
    \qv_t=e^{\mathrm{i}t\boldsymbol\Theta}\mathbf W_Q\hv_t,
    \qquad
    \kv_s=e^{\mathrm{i}s\boldsymbol\Theta}\mathbf W_K\hv_s,
\end{align*}
where $\mathrm{i}$ is the imaginary unit,
and $\qv_t$ and $\kv_s$ denote the rotated query and key
at positions $t$ and $s$.
The attention score is then given by
\begin{align*}
    a_{ts}
    &=\operatorname{Re}(\qv_t^*\kv_s)/\sqrt{d_h}
    =\operatorname{Re}\left(
    (\mathbf W_Q\hv_t)^*
    e^{\mathrm{i}(s-t)\boldsymbol\Theta}
    (\mathbf W_K\hv_s)
    \right)/\sqrt{d_h},
\end{align*}
where $^*$ and $\operatorname{Re}(\cdot)$
denote conjugate transpose and
the real part, respectively.
The term $e^{\mathrm{i}(s-t)\boldsymbol\Theta}$
encodes the \emph{relative position} $s-t$,
allowing the attention score to capture distance information.

Using this notation, \name randomly perturbs the key-side RoPE.
For each target sentence, let $\mathcal B$ contain the prompt and each
preceding sentence as separate context blocks. We independently sample
one offset for every block $b\in\mathcal B$,
\begin{align*}
    \epsilon_b\sim\mathcal N(0,\sigma^2),
    \qquad
    D_b=\operatorname{round}(\epsilon_b),
\end{align*}
where $\sigma$ controls the perturbation scale. Every eligible key position
$s$ in block $b$ uses $d_s=D_b$, while $d_s=0$ in the target sentence.
We then apply these offsets to the context keys while
keeping all query rotations and the target sentence's
key rotations unchanged:
\begin{align*}
    \widetilde{\qv}_t&=\qv_t,\qquad
    \widetilde{\kv}_s
    =e^{\mathrm{i}(s+d_s)\boldsymbol\Theta}\mathbf W_K\hv_s
    =e^{\mathrm{i}d_s\boldsymbol\Theta}\kv_s,\\
    \widetilde a_{ts}
    &=\operatorname{Re}\left(
    (\mathbf W_Q\hv_t)^*
    e^{\mathrm{i}(s+d_s-t)\boldsymbol\Theta}
    (\mathbf W_K\hv_s)
    \right)/\sqrt{d_h}.
\end{align*}
Each block offset is shared across attention layers and resampled at every step. As shown in Fig.~\ref{fig:token-loss}, the resulting focused views generally exhibit higher initial losses than the native view. The perturbation is removed after editing, and the LLM returns to its original RoPE.

The following theorem gives a mean-margin and view-consistency decomposition of \name.
\begin{theorem}[Mean-margin and view-consistency decomposition]
\label{thm:krp:informal}
Let $m_0$ and $m_D$ denote the native and perturbed
margins of a target token,
where $D$ denotes the sampled positional
offsets, with $\ell(m)=\log(1+e^{-m})$.
Under the integrability condition
$\mathbb E_D|m_D|<\infty$, we have
\begin{align*}
    \mathbb E_D[\ell(m_D)]
    =\ell(m_0-\Delta)+\mathcal R,
    \qquad \mathcal R\geq0,
\end{align*}
where $\Delta=m_0-\mathbb E_D[m_D]$ is the change in mean margin
and $\mathcal R$ is an expected Bernoulli KL divergence
that vanishes if and only if the correct-token confidence
is constant across views.
Under additional positive-shortcut assumptions (App.~\ref{app:margin}), $\Delta\geq0$ holds and hence
$\mathbb E_D[\ell(m_D)]\geq\ell(m_0)$.
\end{theorem}

In words, \name aims to learn a fact prediction
that is stable across randomized positional views of the context:
$\mathcal R$ penalizes view-dependent confidence,
and, under a positive shortcut, the term $\Delta$ additionally reduces the mean contextual support.
The complete version and proofs are in App.~\ref{app:theory}.

\subsection{Integrating \name with Existing Editors}
\label{sec:method:editors}

We integrate \name with two editor families: methods that directly optimize
selected parameters~\citep{zhang2024comprehensive,zhou2026context}, and
locate-then-edit methods that write hidden targets into model
weights~\citep{fang2025alphaedit,jiang2025anyedit}.

\textit{Direct Optimization (DO)} editors update selected parameters to minimize the passage NLL.
Assume the passage contains $M$ sentences,
with $\mathcal I_m$ denoting the token positions of
sentence $m$.
\name reorganizes the passage NLL into
sentence-level losses and applies key perturbations
to each sentence's preceding context.
For sentence $m$, the \name-induced loss is given by
\begin{align}
    \ell_m(\theta,\boldsymbol D_m)
    =-\frac{1}{|\mathcal I_m|}
    \sum_{t\in\mathcal I_m}
    \log\pi_\theta\!\left(c_t\mid\xv,\cv_{<t};
    \widetilde{\mathbf K}(\boldsymbol D_m)\right),
    \label{eq:do:sent}
\end{align}
where $\widetilde{\mathbf K}(\boldsymbol D_m)$
denotes the keys randomly perturbed by offsets $\boldsymbol D_m$
in the context preceding sentence $m$,
as defined in Sec.~\ref{sec:method:krp}.
The final objective is a weighted expectation of sentence-level losses:
\begin{align}
    \mathcal L_{\mathrm{\name}}(\theta)
    =\sum_{m=1}^{M}\frac{|\mathcal I_m|}{|\cv|}
    \mathbb E_{\boldsymbol D_m}[\ell_m(\theta,\boldsymbol D_m)].
    \label{eq:do}
\end{align}
As shown in App.~\ref{app:theory},
using zero positional offsets in all context blocks
reduces Eq.~\eqref{eq:do} to Eq.~\eqref{eq:nll}.
We keep the original DO editors'
trainable parameters and optimization settings unchanged.

\textit{Locate-Then-Edit (LTE)} editors first determine how
updating some hidden states can lead to the desired knowledge
update, then write the updates into
selected model layers using either closed-form
solvers or gradient descent on intermediate-state reconstruction
losses~\citep{deng2024unke,jiang2025anyedit,su2025muke}.
Because LTE minimizes the NLL over intermediate states rather than model parameters, adapting \name requires an additional weight-writing step.

We therefore use a two-stage LTE adaptation, following the structure of
\citet{wang2026fable}: sentence-level targets are first written into shallow
MLP layers, after which the base LTE editor learns the complete passage in
deeper layers. Unlike FABLE, \name derives the sentence targets directly from
the editing passage under focused views, without externally synthesized QA pairs.

\emph{Stage 1. Sentence-level LTE.}
For sentence $m$, we learn a trainable residual $\delta_m$
that is added to the MLP outputs at prompt
positions $\mathcal A_m$ at a shallow layer $z$%
\footnote{These anchors include a temporary sentence-specific carrier;
see App.~\ref{app:lte-targets} for construction details.}.
We optimize the residual on the sentence's tokens
within the full teacher-forced passage, with key
perturbations applied to its preceding context:
\begin{align}
    \delta_m^\star
    =\arg\min_{\delta_m}\;
    \mathbb E_{\boldsymbol D_m}
    \left[-\frac{1}{|\mathcal I_m|}
    \sum_{t\in\mathcal I_m}
    \log\pi_{\theta_0, \delta_m}\!\left(c_t\mid\xv,\cv_{<t};
    \widetilde{\mathbf K}(\boldsymbol D_m)\right)\right].
    \label{eq:krp:delta}
\end{align}
After optimizing all sentence-level residuals
$[\delta_1, \ldots, \delta_M]$, we jointly write the resulting
targets into the shallow MLP weights using
a closed-form solver~\citep{fang2025alphaedit}.

\emph{Stage 2. Passage-level LTE.}
Starting from the above model, we apply the base
LTE editor and update deeper layers.
Following \citet{wang2026fable}, we penalize
changes in the hidden states at sentence-level
anchor positions to preserve the Stage 1
edits by adding an MSE term to the original editing objective.
See Appendices~\ref{app:stage1-write}
and~\ref{app:hyperparams} for further implementation details.

\section{Experiments}
\label{sec:experiment}

We evaluate \name with five UKE editors spanning the DO and LTE families.

\subsection{Experimental Setup}
\label{sec:exp:setup}

\textit{Base Models.}
We select two LLMs widely used in UKE studies~\citep{jiang2025anyedit, su2025muke,zhou2026context,wang2026fable}:
\textit{Qwen2.5-7B-Instruct}~\citep{qwen2024qwen25} and
\textit{Llama-3.1-8B-Instruct}~\citep{dubey2024llama3}, which we refer to as \textit{Qwen2.5} and
\textit{Llama3.1}, respectively.

\textit{Tasks.}
We use \textit{UnFine}~\citep{wang2026fable},
which augments existing benchmarks to enable finer-grained evaluation of UKE.
It contains three subsets:
(1) \textit{UnFine-UnKE} extends UnKEBench~\citep{deng2024unke};
(2) \textit{UnFine-CF} and
(3) \textit{UnFine-MQ} extend the CounterFact~\citep{meng2022rome}
and MQuAKE~\citep{zhong2023mquake} subsets of
\citet{wu2024akew}, respectively.
We adopt the single-request setting from prior UKE work~\citep{deng2024unke,wang2026fable}.

\textit{Editing Methods.}
We evaluate \name on five editing methods spanning both DO and LTE families.
Among the DO editors,
\textit{FT-M}~\citep{zhang2024comprehensive} fine-tunes the MLP layer identified by causal tracing,
\textit{LoRA}~\citep{hu2022lora} learns additive low-rank updates to model parameters,
and a simplified variant of
\textit{COIN}$^{\star}$~\citep{zhou2026context}
fine-tunes the model on the passage and regularizes against context reliance%
\footnote{Our COIN$^{\star}$ adopts its original context-alignment objective to test whether \name offers complementary gains against context reliance. Other general-purpose locality-preserving mechanisms are omitted to keep the comparison focused.}.
For the LTE family, we study
\textit{$\mu$KE}~\citep{su2025muke} and
\textit{UnKE}~\citep{deng2024unke},
two representative UKE editors.
All paired comparisons hold supervision fixed to the editing query and passage;
we therefore do not include FABLE, which additionally uses externally synthesized atomic-QA supervision.
Implementation details are provided in App.~\ref{app:setup}.
We report computation overhead of \name in App.~\ref{app:editing-cost}.
\name adds moderate editing overhead but no inference cost.

\textit{Evaluation.}
Following UnFine~\citep{wang2026fable},
we assess both holistic passage recall and atomic-fact recall.
Specifically, for \textit{holistic recall}, we report
\textit{ROUGE-L recall}, which measures lexical overlap
through the longest common subsequence~\citep{lin2004rouge},
under the original editing query (\textit{Orig Q}) and its paraphrase
(\textit{Reph Q}), together with their mean
(\textit{Avg.}).
For \textit{atomic-fact recall},
we separately query the model with atomic questions (\textit{Atom Q})
targeting the individual facts in each passage.
We report \textit{ROUGE-L}, \textit{Hit Rate (HR)}, and
$\mathbf{C}_{\mathrm{LCS}}$.
\textit{HR} measures complete key-phrase recovery, and
$\mathbf{C}_{\mathrm{LCS}}$ measures partial
reference-answer coverage~\citep{wang2026fable}.
We further evaluate \textit{general capability preservation} on six
standard benchmarks after editing on UnFine-UnKE.
More details are
provided in App.~\ref{app:metrics}.

\subsection{Overall Performance}

Tab.~\ref{tab:main-results} reports the editing
performance of \name on the three UnFine benchmarks.

Although base editors maintain high editing performance under Orig Q, their
Atom Q performance is markedly worse, showing that passage-level recall does
not necessarily translate into access to the individual facts.
Applying \name consistently improves atomic-fact recall. For example, atomic
ROUGE-L improves in all 30 scenarios by $4.4$ points on average. Similar
trends hold for HR and $\mathbf{C}_{\mathrm{LCS}}$, which improve in 29 out of
30 scenarios, with the only decreases below two points.
Moreover, the improvement of \name does not come at
the cost of holistic editing performance.
Orig Q scores remain largely unchanged, while Reph Q improves in 19 of the 20
applicable scenarios by $8.4$ points on average, indicating more robust
holistic recall under query reformulation.

The improvements hold across both editor families. For DO editors, atomic
ROUGE-L increases by $10.3\%$ on average, with similar gains in HR and
$\mathbf{C}_{\mathrm{LCS}}$ across every setting. Notably,
COIN$^{\star}$ still benefits from \name despite already including a
context-alignment objective, suggesting that the two mechanisms are
complementary. LTE editors show the same overall trend, with a $12.9\%$
average gain in atomic ROUGE-L and only two minor decreases in the other
metrics. These results show that \name benefits both editing paradigms despite
their different update mechanisms.

\begin{table}[t!]
\centering
\caption{Overall editing performance (mean $\pm$ SE); higher is better. Gains over the base editors are shaded. Missing Reph Q entries are ``--''.}
\label{tab:main-results}
\begingroup
\newcommand{\mainmetric}[2]{#1\hspace{0.25pt}\raisebox{-1.5pt}{{\fontsize{4.5}{4.9}\selectfont\textpm#2}}}
\newcommand{\mainmetricgain}[2]{\cellcolor{PaperAccent!18}\textbf{#1}\hspace{0.25pt}\raisebox{-1.5pt}{{\fontsize{4.5}{4.9}\selectfont\textpm#2}}}
\newcommand{\mainmissing}{--}
\fontsize{6.6}{7.3}\selectfont
\setlength{\tabcolsep}{1.4pt}
\resizebox{0.95\linewidth}{!}{%
\begin{tabular}{c >{\bfseries}r cccccc @{\hspace{7pt}} cccccc}
\toprule[0.4ex]
& & \multicolumn{6}{c}{\bfseries Qwen2.5} & \multicolumn{6}{c}{\bfseries Llama3.1} \\
\cmidrule(lr){3-8}\cmidrule(lr){9-14}
& & \multicolumn{3}{c}{\bfseries Holistic (ROUGE-L)} & \multicolumn{3}{c}{\bfseries Atom Q} & \multicolumn{3}{c}{\bfseries Holistic (ROUGE-L)} & \multicolumn{3}{c}{\bfseries Atom Q} \\
\cmidrule(lr){3-5}\cmidrule(lr){6-8}\cmidrule(lr){9-11}\cmidrule(lr){12-14}
& \textbf{Method} & \textbf{Orig Q} & \textbf{Reph Q} & \textbf{Avg.} & \textbf{ROUGE-L} & \textbf{Hit Rate} & $\mathbf{C}_{\mathrm{LCS}}$ & \textbf{Orig Q} & \textbf{Reph Q} & \textbf{Avg.} & \textbf{ROUGE-L} & \textbf{Hit Rate} & $\mathbf{C}_{\mathrm{LCS}}$ \\
\midrule[0.2ex]
\multirow{10}{*}[-2ex]{\rotatebox[origin=c]{90}{\bfseries UnFine-UnKE}}
& LoRA & \mainmetric{100.0}{0.0} & \mainmetric{95.3}{0.5} & \mainmetric{97.6}{0.3} & \mainmetric{40.8}{0.9} & \mainmetric{40.2}{1.2} & \mainmetric{51.9}{0.9} & \mainmetric{100.0}{0.0} & \mainmetric{96.5}{0.4} & \mainmetric{98.2}{0.2} & \mainmetric{39.2}{0.8} & \mainmetric{43.8}{1.1} & \mainmetric{57.1}{0.9} \\
& +Ours & \mainmetric{99.7}{0.2} & \mainmetricgain{99.1}{0.2} & \mainmetricgain{99.4}{0.2} & \mainmetricgain{53.4}{0.8} & \mainmetricgain{61.8}{1.0} & \mainmetricgain{69.9}{0.8} & \mainmetric{99.9}{0.1} & \mainmetricgain{99.0}{0.2} & \mainmetricgain{99.4}{0.1} & \mainmetricgain{45.4}{0.7} & \mainmetricgain{56.8}{1.0} & \mainmetricgain{66.5}{0.8} \\
\noalign{\vskip 0.5ex}\cdashline{3-14}\noalign{\vskip 0.5ex}
& COIN$^{\star}$ & \mainmetric{100.0}{0.0} & \mainmetric{97.4}{0.4} & \mainmetric{98.7}{0.2} & \mainmetric{36.8}{0.7} & \mainmetric{34.7}{1.0} & \mainmetric{49.0}{0.8} & \mainmetric{100.0}{0.0} & \mainmetric{96.5}{0.5} & \mainmetric{98.3}{0.3} & \mainmetric{36.8}{0.7} & \mainmetric{36.7}{1.1} & \mainmetric{53.2}{0.8} \\
& +Ours & \mainmetric{99.3}{0.2} & \mainmetricgain{98.4}{0.3} & \mainmetricgain{98.9}{0.2} & \mainmetricgain{43.9}{0.7} & \mainmetricgain{47.8}{1.0} & \mainmetricgain{59.2}{0.8} & \mainmetric{100.0}{0.0} & \mainmetricgain{98.2}{0.3} & \mainmetricgain{99.1}{0.2} & \mainmetricgain{41.2}{0.8} & \mainmetricgain{45.7}{1.1} & \mainmetricgain{60.2}{0.8} \\
\noalign{\vskip 0.5ex}\cdashline{3-14}\noalign{\vskip 0.5ex}
& FT-M & \mainmetric{99.9}{0.1} & \mainmetric{78.3}{1.4} & \mainmetric{89.1}{0.7} & \mainmetric{23.2}{0.6} & \mainmetric{16.1}{0.8} & \mainmetric{30.0}{0.7} & \mainmetric{100.0}{0.0} & \mainmetric{96.6}{0.5} & \mainmetric{98.3}{0.2} & \mainmetric{35.3}{0.7} & \mainmetric{34.8}{1.0} & \mainmetric{51.6}{0.8} \\
& +Ours & \mainmetric{99.2}{0.2} & \mainmetricgain{81.5}{1.2} & \mainmetricgain{90.3}{0.6} & \mainmetricgain{26.8}{0.7} & \mainmetricgain{21.1}{0.9} & \mainmetricgain{35.1}{0.8} & \mainmetric{100.0}{0.0} & \mainmetricgain{98.5}{0.2} & \mainmetricgain{99.2}{0.1} & \mainmetricgain{40.9}{0.8} & \mainmetricgain{45.3}{1.1} & \mainmetricgain{59.8}{0.8} \\
\noalign{\vskip 0.5ex}\cdashline{3-14}\noalign{\vskip 0.5ex}
& $\mu$KE & \mainmetric{98.8}{0.3} & \mainmetric{80.8}{1.1} & \mainmetric{89.8}{0.6} & \mainmetric{24.4}{0.6} & \mainmetric{14.8}{0.7} & \mainmetric{33.8}{0.7} & \mainmetric{97.5}{0.5} & \mainmetric{90.9}{0.8} & \mainmetric{94.2}{0.6} & \mainmetric{25.9}{0.7} & \mainmetric{19.6}{0.9} & \mainmetric{37.6}{0.8} \\
& +Ours & \mainmetricgain{99.5}{0.2} & \mainmetricgain{83.0}{1.0} & \mainmetricgain{91.3}{0.5} & \mainmetricgain{29.9}{0.6} & \mainmetricgain{24.2}{0.9} & \mainmetricgain{43.5}{0.8} & \mainmetricgain{99.8}{0.1} & \mainmetricgain{93.0}{0.6} & \mainmetricgain{96.4}{0.3} & \mainmetricgain{29.3}{0.6} & \mainmetricgain{23.6}{0.9} & \mainmetricgain{42.7}{0.7} \\
\noalign{\vskip 0.5ex}\cdashline{3-14}\noalign{\vskip 0.5ex}
& UnKE & \mainmetric{89.6}{0.8} & \mainmetric{66.4}{1.1} & \mainmetric{78.0}{0.8} & \mainmetric{24.8}{0.6} & \mainmetric{14.5}{0.7} & \mainmetric{33.8}{0.7} & \mainmetric{89.7}{0.9} & \mainmetric{82.1}{1.1} & \mainmetric{85.9}{0.9} & \mainmetric{23.3}{0.6} & \mainmetric{16.7}{0.8} & \mainmetric{34.3}{0.8} \\
& +Ours & \mainmetric{87.1}{0.8} & \mainmetricgain{70.6}{1.0} & \mainmetricgain{78.8}{0.8} & \mainmetricgain{29.7}{0.6} & \mainmetricgain{22.7}{0.9} & \mainmetricgain{42.7}{0.8} & \mainmetricgain{93.9}{0.6} & \mainmetricgain{85.4}{0.8} & \mainmetricgain{89.7}{0.6} & \mainmetricgain{29.5}{0.6} & \mainmetricgain{25.6}{0.9} & \mainmetricgain{43.7}{0.8} \\

\midrule[0.2ex]
\multirow{10}{*}[-2ex]{\rotatebox[origin=c]{90}{\bfseries UnFine-CF}}
& LoRA & \mainmetric{99.9}{0.1} & \mainmetric{66.5}{1.5} & \mainmetric{83.2}{0.7} & \mainmetric{56.7}{0.9} & \mainmetric{61.4}{1.0} & \mainmetric{68.4}{0.9} & \mainmetric{100.0}{0.0} & \mainmetric{70.5}{1.3} & \mainmetric{85.2}{0.7} & \mainmetric{54.1}{0.9} & \mainmetric{71.2}{1.0} & \mainmetric{77.8}{0.8} \\
& +Ours & \mainmetric{99.7}{0.1} & \mainmetricgain{95.9}{0.5} & \mainmetricgain{97.8}{0.3} & \mainmetricgain{63.6}{0.9} & \mainmetricgain{76.5}{0.9} & \mainmetricgain{80.8}{0.7} & \mainmetric{99.8}{0.1} & \mainmetricgain{92.3}{0.8} & \mainmetricgain{96.0}{0.4} & \mainmetricgain{56.2}{0.9} & \mainmetricgain{77.2}{0.9} & \mainmetricgain{81.9}{0.7} \\
\noalign{\vskip 0.5ex}\cdashline{3-14}\noalign{\vskip 0.5ex}
& COIN$^{\star}$ & \mainmetric{99.9}{0.1} & \mainmetric{73.1}{1.5} & \mainmetric{86.5}{0.7} & \mainmetric{52.4}{1.0} & \mainmetric{57.7}{1.0} & \mainmetric{65.2}{0.9} & \mainmetric{100.0}{0.0} & \mainmetric{60.7}{1.5} & \mainmetric{80.4}{0.8} & \mainmetric{51.8}{0.9} & \mainmetric{63.3}{1.0} & \mainmetric{71.2}{0.8} \\
& +Ours & \mainmetric{99.2}{0.2} & \mainmetricgain{85.1}{1.2} & \mainmetricgain{92.1}{0.6} & \mainmetricgain{55.1}{0.9} & \mainmetricgain{60.9}{1.0} & \mainmetricgain{68.3}{0.9} & \mainmetric{100.0}{0.1} & \mainmetricgain{78.2}{1.4} & \mainmetricgain{89.1}{0.7} & \mainmetricgain{53.6}{0.9} & \mainmetricgain{67.3}{1.0} & \mainmetricgain{74.6}{0.8} \\
\noalign{\vskip 0.5ex}\cdashline{3-14}\noalign{\vskip 0.5ex}
& FT-M & \mainmetric{100.0}{0.0} & \mainmetric{28.4}{0.9} & \mainmetric{64.2}{0.4} & \mainmetric{39.0}{0.9} & \mainmetric{36.5}{1.0} & \mainmetric{45.8}{1.0} & \mainmetric{100.0}{0.0} & \mainmetric{59.4}{1.5} & \mainmetric{79.7}{0.8} & \mainmetric{50.9}{0.9} & \mainmetric{62.1}{1.0} & \mainmetric{70.4}{0.8} \\
& +Ours & \mainmetric{98.9}{0.3} & \mainmetricgain{37.4}{1.0} & \mainmetricgain{68.2}{0.5} & \mainmetricgain{44.1}{0.9} & \mainmetricgain{42.5}{1.1} & \mainmetricgain{51.8}{1.0} & \mainmetric{100.0}{0.0} & \mainmetricgain{78.0}{1.4} & \mainmetricgain{89.0}{0.7} & \mainmetricgain{53.2}{0.9} & \mainmetricgain{67.1}{1.0} & \mainmetricgain{74.4}{0.8} \\
\noalign{\vskip 0.5ex}\cdashline{3-14}\noalign{\vskip 0.5ex}
& $\mu$KE & \mainmetric{98.7}{0.3} & \mainmetric{31.5}{0.7} & \mainmetric{65.1}{0.4} & \mainmetric{33.1}{0.8} & \mainmetric{30.1}{0.9} & \mainmetric{41.9}{0.9} & \mainmetric{98.1}{0.5} & \mainmetric{50.7}{1.2} & \mainmetric{74.4}{0.7} & \mainmetric{38.8}{0.8} & \mainmetric{42.3}{1.0} & \mainmetric{53.2}{0.9} \\
& +Ours & \mainmetricgain{99.6}{0.1} & \mainmetricgain{47.4}{0.9} & \mainmetricgain{73.5}{0.5} & \mainmetricgain{38.4}{0.8} & \mainmetricgain{38.5}{1.0} & \mainmetricgain{50.2}{0.9} & \mainmetricgain{100.0}{0.0} & \mainmetric{49.3}{1.1} & \mainmetricgain{74.6}{0.6} & \mainmetricgain{39.3}{0.8} & \mainmetric{40.5}{1.0} & \mainmetric{51.8}{0.9} \\
\noalign{\vskip 0.5ex}\cdashline{3-14}\noalign{\vskip 0.5ex}
& UnKE & \mainmetric{85.8}{0.9} & \mainmetric{30.0}{0.7} & \mainmetric{57.9}{0.6} & \mainmetric{31.7}{0.8} & \mainmetric{29.5}{1.0} & \mainmetric{41.2}{0.9} & \mainmetric{94.2}{0.7} & \mainmetric{41.3}{1.1} & \mainmetric{67.8}{0.7} & \mainmetric{35.7}{0.8} & \mainmetric{37.0}{1.0} & \mainmetric{48.3}{0.9} \\
& +Ours & \mainmetricgain{89.5}{0.8} & \mainmetricgain{44.0}{0.8} & \mainmetricgain{66.8}{0.6} & \mainmetricgain{38.6}{0.9} & \mainmetricgain{38.2}{1.0} & \mainmetricgain{50.0}{0.9} & \mainmetric{93.9}{0.6} & \mainmetricgain{47.2}{1.1} & \mainmetricgain{70.6}{0.7} & \mainmetricgain{38.7}{0.8} & \mainmetricgain{40.7}{1.0} & \mainmetricgain{51.4}{0.9} \\

\midrule[0.2ex]
\multirow[c]{10}{*}[-2ex]{%
  \rotatebox[origin=c]{90}{\bfseries UnFine-MQ}}%
& LoRA & \mainmetric{100.0}{0.0} & \mainmissing & \mainmetric{100.0}{0.0} & \mainmetric{55.7}{1.1} & \mainmetric{58.2}{1.3} & \mainmetric{65.9}{1.1} & \mainmetric{100.0}{0.0} & \mainmissing & \mainmetric{100.0}{0.0} & \mainmetric{50.2}{1.1} & \mainmetric{65.8}{1.3} & \mainmetric{72.8}{1.1} \\
& +Ours & \mainmetric{99.4}{0.2} & \mainmissing & \mainmetric{99.4}{0.2} & \mainmetricgain{63.3}{1.1} & \mainmetricgain{72.3}{1.1} & \mainmetricgain{78.1}{0.9} & \mainmetric{100.0}{0.0} & \mainmissing & \mainmetric{100.0}{0.0} & \mainmetricgain{54.7}{1.1} & \mainmetricgain{73.4}{1.1} & \mainmetricgain{78.4}{1.0} \\
\noalign{\vskip 0.5ex}\cdashline{3-14}\noalign{\vskip 0.5ex}
& COIN$^{\star}$ & \mainmetric{100.0}{0.0} & \mainmissing & \mainmetric{100.0}{0.0} & \mainmetric{53.0}{1.1} & \mainmetric{55.1}{1.3} & \mainmetric{62.8}{1.1} & \mainmetric{100.0}{0.0} & \mainmissing & \mainmetric{100.0}{0.0} & \mainmetric{49.2}{1.1} & \mainmetric{57.0}{1.3} & \mainmetric{65.1}{1.2} \\
& +Ours & \mainmetric{99.6}{0.2} & \mainmissing & \mainmetric{99.6}{0.2} & \mainmetricgain{56.0}{1.1} & \mainmetricgain{60.2}{1.2} & \mainmetricgain{67.8}{1.1} & \mainmetric{100.0}{0.0} & \mainmissing & \mainmetric{100.0}{0.0} & \mainmetricgain{51.2}{1.1} & \mainmetricgain{62.8}{1.2} & \mainmetricgain{70.0}{1.1} \\
\noalign{\vskip 0.5ex}\cdashline{3-14}\noalign{\vskip 0.5ex}
& FT-M & \mainmetric{100.0}{0.0} & \mainmissing & \mainmetric{100.0}{0.0} & \mainmetric{42.0}{1.1} & \mainmetric{39.3}{1.3} & \mainmetric{48.1}{1.1} & \mainmetric{100.0}{0.0} & \mainmissing & \mainmetric{100.0}{0.0} & \mainmetric{48.5}{1.1} & \mainmetric{57.4}{1.3} & \mainmetric{65.4}{1.2} \\
& +Ours & \mainmetric{97.8}{0.5} & \mainmissing & \mainmetric{97.8}{0.5} & \mainmetricgain{46.1}{1.1} & \mainmetricgain{44.0}{1.3} & \mainmetricgain{52.6}{1.2} & \mainmetric{100.0}{0.0} & \mainmissing & \mainmetric{100.0}{0.0} & \mainmetricgain{50.8}{1.2} & \mainmetricgain{61.6}{1.3} & \mainmetricgain{68.6}{1.1} \\
\noalign{\vskip 0.5ex}\cdashline{3-14}\noalign{\vskip 0.5ex}
& $\mu$KE & \mainmetric{98.5}{0.5} & \mainmissing & \mainmetric{98.5}{0.5} & \mainmetric{35.7}{0.9} & \mainmetric{33.1}{1.2} & \mainmetric{45.1}{1.0} & \mainmetric{98.5}{0.5} & \mainmissing & \mainmetric{98.5}{0.5} & \mainmetric{35.9}{1.0} & \mainmetric{40.5}{1.2} & \mainmetric{50.0}{1.1} \\
& +Ours & \mainmetricgain{99.6}{0.2} & \mainmissing & \mainmetricgain{99.6}{0.2} & \mainmetricgain{38.1}{1.0} & \mainmetricgain{38.1}{1.1} & \mainmetricgain{48.4}{1.0} & \mainmetricgain{99.8}{0.2} & \mainmissing & \mainmetricgain{99.8}{0.2} & \mainmetricgain{40.8}{1.0} & \mainmetricgain{41.5}{1.1} & \mainmetricgain{52.3}{1.0} \\
\noalign{\vskip 0.5ex}\cdashline{3-14}\noalign{\vskip 0.5ex}
& UnKE & \mainmetric{84.9}{1.2} & \mainmissing & \mainmetric{84.9}{1.2} & \mainmetric{33.6}{0.9} & \mainmetric{32.8}{1.1} & \mainmetric{43.8}{1.0} & \mainmetric{92.3}{1.0} & \mainmissing & \mainmetric{92.3}{1.0} & \mainmetric{36.1}{1.0} & \mainmetric{38.2}{1.3} & \mainmetric{49.3}{1.2} \\
& +Ours & \mainmetricgain{85.6}{1.1} & \mainmissing & \mainmetricgain{85.6}{1.1} & \mainmetricgain{35.5}{0.9} & \mainmetricgain{35.0}{1.1} & \mainmetricgain{46.2}{1.0} & \mainmetric{92.0}{0.8} & \mainmissing & \mainmetric{92.0}{0.8} & \mainmetricgain{40.1}{1.0} & \mainmetricgain{41.1}{1.2} & \mainmetricgain{51.8}{1.1} \\
\bottomrule[0.4ex]
\end{tabular}
}
\endgroup
\end{table}

\subsection{Atomic-Fact Recall by Passage Position}

\begin{figure}[tb]
\centering
\includegraphics[width=0.95\linewidth]{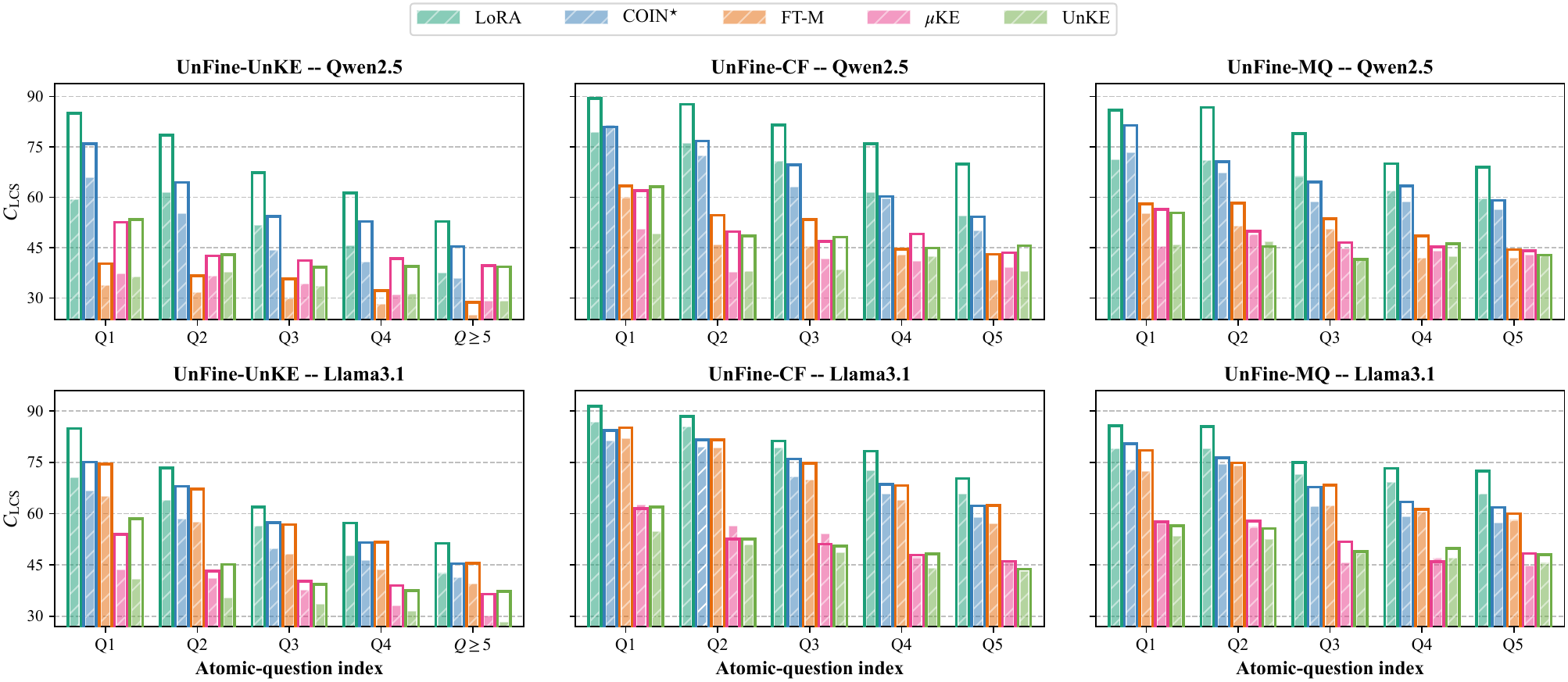}
\caption{
Atomic-fact recall by question position; higher is better.
Solid and transparent bars show performance with and
without \name. Unfilled area represents the performance gap.
}
\label{fig:qidx-depth}
\end{figure}

Fig.~\ref{fig:qidx-depth} reports average atomic-fact recall
by question index, which approximately tracks the corresponding fact's
position; complete results are shown in App.~\ref{app:atomic-positions}.

Consistent with \citet{zhou2026context}, 
atomic-fact recall declines toward
later question positions.
Together with the lower pre-edit losses observed at later passage positions (Fig.~\ref{fig:token-loss}), these observations support
our account of context-induced difficulty underestimation.

Moreover, 
applying \name raises recall for most
editor--position pairs across the passage, with gains observed across all five editors, three
benchmarks, and both LLM backbones.
The improvement persists at later positions: the average gain at late 
$Q\!\geq\!5$ is approximately $6.0$ points, comparable to
the gains observed
from $Q2$ -- $Q4$.
The comparable gains at later positions show that focused
views remain effective when longer preceding contexts provide stronger support
under teacher forcing.

\subsection{General Capability Preservation}

\begin{figure}[tbp!]
\centering
\includegraphics[width=0.95\linewidth]{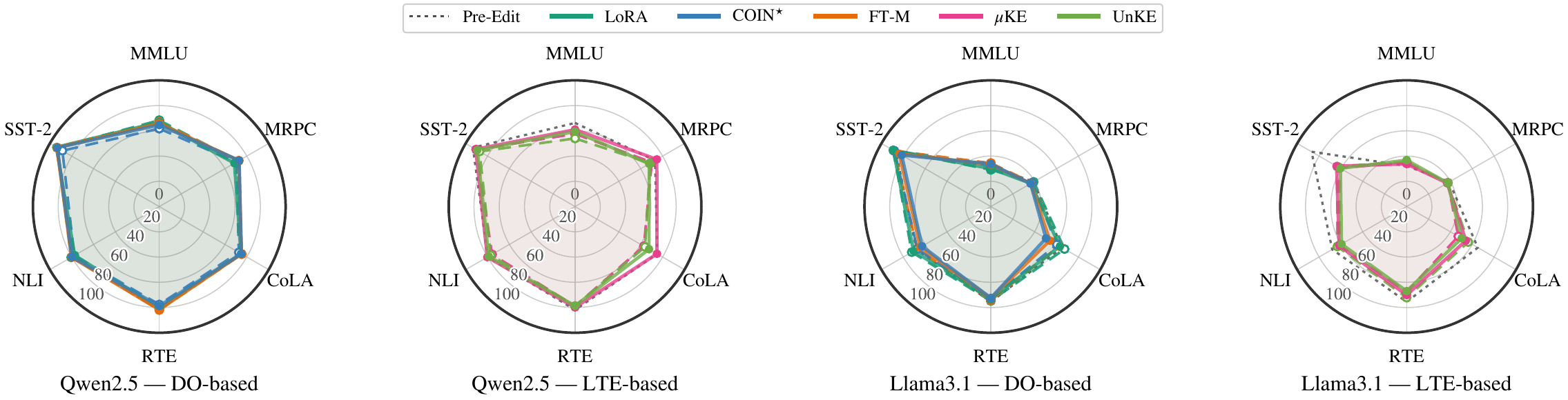}
\caption{
General capability (likelihood F1) after editing on UnFine-UnKE.
Solid and dashed lines show performance with and without \name, respectively.
}
\label{fig:general-capability}
\end{figure}

We evaluate the edited models on six general-purpose language-understanding
and reasoning tasks; dataset details are in App.~\ref{app:general-capability}. We summarize the performance in Fig.~\ref{fig:general-capability};  
numbers are in
Tab.~\ref{tab:general-capability-scores}.

Across the ten editor--backbone combinations, macro-average
F1 changes relative to the base editors by at most $4.3$ points in either direction. Performance remains
stable or improves on Qwen2.5, with the clearest gains under LTE. On Llama3.1,
$\mu$KE improves, whereas the DO editors and UnKE show moderate aggregate
decreases driven by task-specific variation. The results therefore support
broad preservation rather than uniform gains in general capability.

\subsection{Ablation Studies}

\begin{wraptable}{r}{0.49\textwidth}
\centering
\captionsetup{
    font=footnotesize,
    justification=raggedright,
    singlelinecheck=false
}
\vspace{-1.3\baselineskip}

\caption{
Ablations on Qwen2.5 using a subset of UnFine-UnKE.
}
\label{tab:ablation}

\begingroup
\footnotesize
\renewcommand{\arraystretch}{1.08}

\begin{adjustbox}{max width=\linewidth}
\begin{tabular}{@{}>{\bfseries}r ccccc@{}}
\toprule[0.4ex]
& \multicolumn{2}{c}{\textbf{Holistic (ROUGE-L)}}
& \multicolumn{3}{c}{\textbf{Atom Q}} \\
\cmidrule(lr){2-3}
\cmidrule(lr){4-6}

\textbf{Method}
& \textbf{Orig Q}
& \textbf{Reph Q}
& \textbf{ROUGE-L}
& \textbf{Hit Rate}
& $\mathbf{C}_{\mathrm{LCS}}$ \\
\midrule

FT-M (DO)
& 100.0 & 81.3 & 20.8 & 15.9 & 28.8 \\

\noalign{\vskip 0.3ex}\cdashline{2-6}\noalign{\vskip 0.3ex}

No Pert.
& 100.0 & 81.2 & 21.0 & 15.5 & 29.2 \\

ID Pert.
& 100.0 & 83.3 & 22.1 & 16.1 & 30.3 \\

\noalign{\vskip 0.3ex}\cdashline{2-6}\noalign{\vskip 0.3ex}

+Ours
& 99.4 & 84.6 & 25.1 & 19.4 & 33.3 \\

\midrule[0.2ex]

$\mu$KE (LTE)
& 98.5 & 82.1 & 24.6 & 15.7 & 35.9 \\
\noalign{\vskip 0.3ex}\cdashline{2-6}\noalign{\vskip 0.3ex}

No Pert.
& 99.6 & 70.3 & 25.1 & 17.9 & 37.5 \\

ID Pert.
& 99.6 & 76.2 & 26.8 & 21.1 & 40.7 \\

Sent.-only
& 40.1 & 38.5 & 25.1 & 21.4 & 40.3 \\

Rev. order
& 48.9 & 48.9 & 29.1 & 22.1 & 44.7 \\
\noalign{\vskip 0.3ex}\cdashline{2-6}\noalign{\vskip 0.3ex}

+Ours
& 99.9 & 85.5 & 29.5 & 24.2 & 45.2 \\

\bottomrule[0.4ex]
\end{tabular}
\end{adjustbox}

\endgroup
\vspace{-1.3\baselineskip}
\end{wraptable}

We ablate \name with FT-M (DO) and $\mu$KE (LTE). Tab.~\ref{tab:ablation} examines the perturbation scheme shared by both paradigms and the two-stage writing strategy
used for LTE.

\textit{Perturbation scheme.}
Removing perturbation ({No Pert.}) but retaining the \name objective reduces atomic ROUGE-L for both editors; FT-M, for example, drops from $25.1$ to $21.0$. Perturbing
shared position IDs ({ID Pert.}), as in \citet{zhu2024pose} and
\citet{li2026shuffle}, partially recovers the loss but trails the key-only
design for both editors, supporting our asymmetric key perturbation.
For space, we defer additional sensitivity analysis on the perturbation scale $\sigma$ to App.~\ref{app:sigma-sensitivity}, where \name exhibits notable robustness to the choice of $\sigma$ over editor-specific ranges.

\textit{LTE writing strategy.}
The \emph{Sent.-only} variant removes passage-level optimization, whereas
\emph{Rev. order} retains both stages but reverses them. Both variants sharply
reduce holistic recall; for example, removing the
passage-level stage lowers Orig Q/Reph Q from $99.9/85.5$ to $40.1/38.5$.
Reversing the order similarly hurts holistic recall while nearly preserving
atomic ROUGE-L ($29.1$ vs.\ $29.5$), showing that strong atomic recall alone
does not preserve the edited passage as a coherent whole.
The passage-level
stage and its placement after sentence-level optimization are therefore both
important. 

\subsection{Case Studies}

We conclude this section with case studies on \name. 
Tab.~\ref{tab:d3-cases} and App.~\ref{app:d3-cases} present
representative failures across editors and backbones. Base editors
may hallucinate or revert to prior knowledge, resulting in UKE failures. Equipped with \name, the edited models instead recover the
target atomic facts and give fluent answers that remain relevant to the
corresponding atomic questions.

\begin{table}[H]
\centering
\caption{Atomic-fact recall case studies, with errors highlighted.}
\label{tab:d3-cases}
\begingroup
\newcommand{\casewrong}[1]{\textcolor{red!80!black}{\textbf{#1}}}
\newcommand{\caseright}[1]{\textcolor{green!50!black}{\textbf{#1}}}
\fontsize{6.8}{7.6}\selectfont
\setlength{\tabcolsep}{2.2pt}
\renewcommand{\arraystretch}{1.08}
\begin{tabular}{@{}
>{\raggedright\arraybackslash}m{0.078\linewidth}
>{\raggedleft\arraybackslash}m{0.055\linewidth}
>{\raggedright\arraybackslash}m{0.085\linewidth}
>{\raggedright\arraybackslash}m{0.238\linewidth}
>{\raggedright\arraybackslash}m{0.238\linewidth}
>{\raggedright\arraybackslash}m{0.238\linewidth}@{}}
\toprule[0.4ex]
\textbf{Model} &
\textbf{Method} &
\textbf{Dataset} &
\textbf{Atomic query (target)} &
\textbf{Base-editor output} &
\textbf{+Ours output} \\
\midrule[0.2ex]

\multirow[c]{3}{*}[-3ex]{Qwen2.5} &
FT-M &
UnFine-CF &
What year was Charles Alfred Pillsbury born? (1842) &
\casewrong{1850} &
\caseright{1842} \\

\noalign{\vskip 0.5ex}
\cdashline{2-6}
\noalign{\vskip 0.5ex}

&
COIN$^{\star}$ &
UnFine-UnKE &
At what age was Henry Feilden elected to the local council? (25) &
Henry Feilden was elected to the local council at the age of
\casewrong{35}. &
Henry Feilden was elected to the local council at the age of
\caseright{25}. \\

\noalign{\vskip 0.5ex}
\cdashline{2-6}
\noalign{\vskip 0.5ex}

&
$\mu$KE &
UnFine-CF &
What region is Sheffield located in? (South Yorkshire) &
Sheffield is a city located in the
\casewrong{northern region of England}. &
Sheffield, located in
\caseright{South Yorkshire}, England. \\

\midrule[0.2ex]

\multirow[c]{2}{*}[-3ex]{Llama3.1} &
LoRA &
UnFine-UnKE &
How many people does Herlyn Espinal employ in her company?
(over 50 people) &
Herlyn Espinal is a politician and
\casewrong{doesn't employ anyone} \ldots &
Herlyn Espinal is the founder and CEO \ldots{} employing
\caseright{over 50 people} \ldots \\

\noalign{\vskip 0.5ex}
\cdashline{2-6}
\noalign{\vskip 0.5ex}

&
UnKE &
UnFine-MQ &
Which year was SC Atl\'etico founded? (1903) &
SC Atl\'etico, a Portuguese sports club, was founded in
\casewrong{1919}, marking \ldots &
SC Atl\'etico, a renowned sports club based in Madrid, Spain,
was founded in \caseright{1903}. \\

\bottomrule[0.4ex]
\end{tabular}
\endgroup
\end{table}

\section{Related Work}
\label{sec:related_work}

\textit{Knowledge Editing}
aims to update specific factual knowledge in a pretrained
model and preserve unrelated knowledge and the model's general
capabilities~\citep{wang2023knowledge,zhang2024comprehensive}.
Optimization-based approaches learn model updates with the
editing objective, including selected layer optimization~\citep{zhang2024comprehensive} or learning multiple copies~\citep{wang2024wise}. Parameter-efficient methods have also been applied~\citep{hu2022lora,wang2024roselora,wu2024reft,liu2025unlocking}.
On the other hand, Locate-Then-Edit (LTE) methods first identify where a fact is stored,
then compute the weight change often in closed form~\citep{meng2022rome,meng2023memit,li2023pmet,fang2025alphaedit}.
\textit{Unstructured Knowledge Editing} (UKE) replaces the curated triple with a free-form
passage~\citep{wu2024akew,yang2025mirage}. UnKE makes transformer blocks
editable~\citep{deng2024unke}, AnyEdit decomposes passages into sequential
chunks~\citep{jiang2025anyedit}, and $\mu$KE coordinates sentence-level updates
through a nested editing structure~\citep{su2025muke}.
COIN identifies context reliance during editing and mitigates it by
aligning predictions under global and local
contexts~\citep{zhou2026context}.
FABLE uses externally synthesized atomic-QA supervision to write shallow-layer
anchors before fitting the passage holistically, and introduces UnFine to
measure atomic-fact access~\citep{wang2026fable}.
In concurrent work, HPSE uses hybrid rollouts to recover missing facts with token-level
feedback from a privileged in-context state~\citep{liu2026hpse}.
In contrast, our \name targets the learning signal by disrupting contextual
support that can suppress gradients before a fact becomes independently
retrievable.

\textit{Positional Encoding.}
RoPE encodes relative positions by rotating queries and
keys~\citep{su2021roformer} and is widely used in modern
LLMs~\citep{dubey2024llama3,qwen2024qwen25,liu2024deepseek}.
Long-context adaptation also manipulates positions:
\citet{peng2024yarn} scale RoPE frequencies and attention,
\citet{zhu2024pose} shift the indices of entire chunks, and
\citet{li2026shuffle} shift suffix indices and distill the perturbed
view from native-view predictions.
These methods perturb positions to improve the model's tolerance to
positional variation. \name instead uses perturbation for a fundamentally
different purpose: by randomizing only preceding-context keys and
supervising with ground truth rather than native-view predictions, it
breaks contextual shortcuts and makes the edited fact robust across
contexts.
\section{Conclusion}
\label{sec:conclusion}

In this work, we identify \emph{context-induced difficulty underestimation} in UKE and propose \name, which constructs focused positional views to make contextual shortcuts less reliable during editing.
Across five editors, two LLM backbones, and three benchmarks, \name broadly improves fine-grained fact recall and retains holistic recall.
Encouraged by these consistent improvements, we plan to extend \name to handle correlated, continually evolving unstructured knowledge, where new updates may reinforce, revise, or invalidate earlier ones.
By tracking these dependencies and selectively preserving or revising previously edited knowledge, we aim to pave the way toward practical UKE in continuously changing real-world environments.

\bibliography{main}

\appendix
\newpage
\appendix

\section{Omitted Theoretical Analysis}
\label{app:theory}

In this section we develop the analysis omitted in the main body due to page limits.

\subsection{Token-Weighted Aggregation Preserves the Passage Objective}

Let the editing passage contain $M$ sentences with token spans
$\mathcal I_m$, $m=1,\ldots,M$, where
$n_m=|\mathcal I_m|$ and $N=\sum_{m=1}^M n_m=|\cv|$. Without perturbation, the
token-averaged loss on sentence $m$ is
\begin{align*}
    \ell_m(\theta,\mathbf 0)
    =
    -\frac{1}{n_m}
    \sum_{t\in\mathcal I_m}
    \log\pi_\theta(c_t\mid\xv,\cv_{<t}).
\end{align*}

\begin{proposition}[Token-weighted sentence aggregation]
\label{prop:app:sampling}
As the sentence spans form a disjoint partition of the passage, weighting
each sentence loss by $n_m/N$ recovers the passage NLL exactly:
\begin{align}
    \sum_{m=1}^{M}\frac{n_m}{N}\ell_m(\theta,\mathbf 0)
    =
    -\frac{1}{N}
    \sum_{t=1}^{N}
    \log\pi_\theta(c_t\mid\xv,\cv_{<t})
    =
    \mathcal{L}_{\mathrm{NLL}}(\theta).
    \label{eq:app:unbiased}
\end{align}
\end{proposition}

\begin{proof}
We expand the weighted sum. 
By definition:
\begin{align*}
\sum_{m=1}^{M}\frac{n_m}{N}\ell_m(\theta,\mathbf 0)
&=\frac{1}{N}\sum_{m=1}^{M}n_m\ell_m(\theta,\mathbf 0)\\
&=-\frac{1}{N}\sum_{m=1}^{M}n_m\frac{1}{n_m}
  \sum_{t\in\mathcal I_m}\log\pi_\theta(c_t\mid\xv,\cv_{<t})\\
&=-\frac{1}{N}\sum_{m=1}^{M}\sum_{t\in\mathcal I_m}
\log\pi_\theta(c_t\mid\xv,\cv_{<t})\\
&\overset{(a)}{=}-\frac{1}{N}\sum_{t=1}^{N}\log\pi_\theta(c_t\mid\xv,\cv_{<t})\\
&=\mathcal L_{\mathrm{NLL}}(\theta).
\end{align*}
The step $(a)$ uses the fact that the disjoint sentence spans partition the passage. This completes our proof. 
\end{proof}

\subsection{Blockwise RoPE-Frequency Contraction}

Following the RoPE convention~\citep{su2021roformer,peng2024yarn},
we focus on a single attention head of even width $d_h=2R$.
Write $\bar{\qv}_t=\mathbf W_Q\hv_t$ and
$\bar{\kv}_s=\mathbf W_K\hv_s$ for the pre-RoPE query and key
in real coordinates, each partitioned into $R$ two-dimensional
components. RoPE at position $p$ takes the form
\begin{align*}
    \mathbf R_p
    =
    \operatorname{diag}
    \left(
        \operatorname{Rot}(\omega_1p),\ldots,
        \operatorname{Rot}(\omega_Rp)
    \right),
\end{align*}
where $\operatorname{Rot}(\alpha)$ denotes the $2\times2$
rotation matrix by angle $\alpha$.
Using the fact that
$\operatorname{Rot}(\alpha)^\top=\operatorname{Rot}(-\alpha)$
and
$\operatorname{Rot}(\alpha)\operatorname{Rot}(\beta)
=\operatorname{Rot}(\alpha+\beta)$,
we have
\begin{align*}
    \mathbf R_t^\top\mathbf R_s
    =
    \operatorname{diag}
    \left(
        \operatorname{Rot}(\omega_1(s-t)),\ldots,
        \operatorname{Rot}(\omega_R(s-t))
    \right).
\end{align*}
Therefore, the native query--key logit can be written as
\begin{align*}
a_{ts}
&=
\frac{(\mathbf R_t\bar{\qv}_t)^\top
(\mathbf R_s\bar{\kv}_s)}{\sqrt{d_h}}
=
\frac{1}{\sqrt{d_h}}
\sum_{r=1}^{R}
\bar{\qv}_{t,r}^{\top}
\operatorname{Rot}(\omega_r(s-t))
\bar{\kv}_{s,r}.
\end{align*}

Now fix a target sentence $m$.
As outlined in Sec.~\ref{sec:method:krp}, its preceding context is
partitioned into separate blocks: the prompt and each preceding
sentence. Let $\mathcal P$ denote the prompt token positions and
$\mathcal I_j$ the token positions of sentence $j$. The set of
context blocks for target sentence $m$ is
\begin{align*}
    \mathcal B_m
    \triangleq
    \{\mathcal P,\mathcal I_1,\ldots,\mathcal I_{m-1}\}.
\end{align*}
For each block $b\in\mathcal B_m$, let $D_{mb}$ denote the
scalar positional offset shared by its eligible key positions.
Thus, $m$ indexes the target sentence, and $b$ identifies a
block in that sentence's preceding context. We collect these
block offsets into
\begin{align*}
    \boldsymbol D_m
    \triangleq
    (D_{mb})_{b\in\mathcal B_m},
\end{align*}
which is the offset vector used in
Eqs.~\eqref{eq:do:sent} and~\eqref{eq:do}.

As will be specified in App.~\ref{app:offsets} shortly, offsets are sampled independently across blocks at the same scale $\sigma$.
Each sampled offset is shared across attention layers and
resampled at every optimization step.
An eligible key position $s\in b$ uses $d_s=D_{mb}$.
The target sentence's key positions and the key at the readout
position of its first token use $d_s=0$, and all query rotations
remain unchanged.

For an eligible key position $s\in b$, the perturbed logit is
\begin{align}
    \widetilde a_{ts}(D_{mb})
    =
    \frac{1}{\sqrt{d_h}}
    \sum_{r=1}^{R}
    \bar{\qv}_{t,r}^{\top}
    \operatorname{Rot}
    \left(\omega_r(s-t+D_{mb})\right)
    \bar{\kv}_{s,r}.
    \label{eq:app:krp-score}
\end{align}
We show that the logit depends
only on the offset $D_{mb}$ of the block containing $s$,
whereas $\boldsymbol D_m$ specifies the perturbations for
the entire sentence-level forward pass.

\begin{proposition}[Blockwise spectral contraction]
\label{prop:app:contraction}
Fix a target sentence $m$, a context block $b\in\mathcal B_m$,
and the pre-RoPE query and key representations in
Eq.~\eqref{eq:app:krp-score}.
Suppose that $D_{mb}$ has a distribution symmetric about zero,
and define
\begin{align*}
    \rho_{mbr}
    \triangleq
    \mathbb E_{D_{mb}}
    \left[\cos(\omega_rD_{mb})\right].
\end{align*}
Then
\begin{align}
    \mathbb E_{D_{mb}}
    \left[\widetilde a_{ts}(D_{mb})\right]
    &=
    \frac{1}{\sqrt{d_h}}
    \sum_{r=1}^{R}
    \rho_{mbr}
    \bar{\qv}_{t,r}^{\top}
    \operatorname{Rot}(\omega_r(s-t))
    \bar{\kv}_{s,r},
    \qquad
    |\rho_{mbr}|\leq1.
    \label{eq:app:expected-score}
\end{align}

For the rounded Gaussian used by \name, let
$D_{mb}=\operatorname{round}(G_{mb})$ with
$G_{mb}\sim\mathcal N(0,\sigma^2)$.
For $\sigma>0$, the exact probability mass function is
\begin{align}
    \Pr(D_{mb}=k)
    &=
    \Phi\!\left(\frac{k+\frac12}{\sigma}\right)
    -
    \Phi\!\left(\frac{k-\frac12}{\sigma}\right),
    \qquad k\in\mathbb Z,
    \label{eq:app:rounded-pmf}
\end{align}
where $\Phi$ denotes the standard Gaussian cumulative
distribution function. The corresponding spectral multiplier is
\begin{align}
    \rho_{mbr}^{\mathrm{round}}
    =
    \sum_{k\in\mathbb Z}
    \Pr(D_{mb}=k)\cos(\omega_rk).
    \label{eq:app:rounded-rho}
\end{align}
Moreover, $|\rho_{mbr}^{\mathrm{round}}|<1$ whenever
$\sigma>0$ and $\omega_r\notin2\pi\mathbb Z$.

For the continuous Gaussian relaxation
$D_{mb}\sim\mathcal N(0,\sigma^2)$, the multiplier simplifies to
\begin{align}
    \rho_{mbr}^{\mathrm{cont}}
    =
    \exp\left(-\frac{\sigma^2\omega_r^2}{2}\right),
    \label{eq:app:gaussian-rho}
\end{align}
which decreases monotonically with $\sigma$ for every nonzero
RoPE frequency. At $\sigma=0$, both offset distributions
concentrate at zero and their multipliers equal $1$.
\end{proposition}

\begin{proof}
We first establish the result for symmetric offsets, then
specialize it to the rounded Gaussian used by \name and
its continuous Gaussian relaxation.

\emph{General result for symmetric offsets.}
Since the distribution of $D_{mb}$ is symmetric about zero,
$D_{mb}$ and $-D_{mb}$ have the same distribution.
Therefore,
\begin{align*}
    \mathbb E\sin(\omega_rD_{mb})
    &=
    \mathbb E\sin(-\omega_rD_{mb})\\
    &=
    -\mathbb E\sin(\omega_rD_{mb}),
\end{align*}
which implies $\mathbb E\sin(\omega_rD_{mb})=0$.
Hence, the expected rotation is a scalar matrix:
\begin{align*}
    \mathbb E\operatorname{Rot}(\omega_rD_{mb})
    &=
    \begin{pmatrix}
        \mathbb E\cos(\omega_rD_{mb})
        &-\mathbb E\sin(\omega_rD_{mb})\\
        \mathbb E\sin(\omega_rD_{mb})
        &\mathbb E\cos(\omega_rD_{mb})
    \end{pmatrix}\\
    &=\rho_{mbr}\mathbf I_2.
\end{align*}
Using the composition rule for rotations,
\begin{align*}
    \operatorname{Rot}\left(\omega_r(s-t+D_{mb})\right)
    =
    \operatorname{Rot}\left(\omega_r(s-t)\right)
    \operatorname{Rot}(\omega_rD_{mb}).
\end{align*}
Substituting this identity into
Eq.~\eqref{eq:app:krp-score}, with the pre-RoPE
representations held fixed, gives
\begin{align*}
    \mathbb E_{D_{mb}}
    \left[\widetilde a_{ts}(D_{mb})\right]
    &=
    \frac{1}{\sqrt{d_h}}
    \sum_{r=1}^{R}
    \bar{\qv}_{t,r}^{\top}
    \operatorname{Rot}(\omega_r(s-t))
    \mathbb E\left[
        \operatorname{Rot}(\omega_rD_{mb})
    \right]
    \bar{\kv}_{s,r}\\
    &=
    \frac{1}{\sqrt{d_h}}
    \sum_{r=1}^{R}
    \rho_{mbr}
    \bar{\qv}_{t,r}^{\top}
    \operatorname{Rot}(\omega_r(s-t))
    \bar{\kv}_{s,r}.
\end{align*}
Moreover,
\begin{align*}
    |\rho_{mbr}|
    &=
    \left|\mathbb E\left[\cos(\omega_rD_{mb})\right]\right|\\
    &\leq
    \mathbb E\left[|\cos(\omega_rD_{mb})|\right]\\
    &\leq1.
\end{align*}
Consequently, the operator norm satisfies
\begin{align*}
    \left\|
        \mathbb E\left[\operatorname{Rot}(\omega_rD_{mb})\right]
    \right\|_{\mathrm{op}}
    =
    |\rho_{mbr}|
    \leq1.
\end{align*}

\emph{Rounded Gaussian case.}
For $\sigma>0$ and $k\in\mathbb Z$, nearest-integer
rounding gives
\begin{align*}
    \Pr(D_{mb}=k)
    &=
    \int_{k-1/2}^{k+1/2}
    \frac{1}{\sqrt{2\pi}\sigma}
    \exp\left(-\frac{g^2}{2\sigma^2}\right)\,dg\\
    &=
    \Phi\left(\frac{k+1/2}{\sigma}\right)
    -
    \Phi\left(\frac{k-1/2}{\sigma}\right).
\end{align*}
Since $D_{mb}$ is integer-valued,
\begin{align*}
    \rho_{mbr}^{\mathrm{round}}
    &=
    \mathbb E\left[\cos(\omega_rD_{mb})\right]\\
    &=
    \sum_{k\in\mathbb Z}
    \Pr(D_{mb}=k)\cos(\omega_rk).
\end{align*}
The series is absolutely convergent because
\begin{align*}
    \sum_{k\in\mathbb Z}
    \Pr(D_{mb}=k)|\cos(\omega_rk)|
    \leq1.
\end{align*}

For strict contraction, the strictly positive Gaussian density
at $\sigma>0$ gives
$\Pr(D_{mb}=0)>0$ and $\Pr(D_{mb}=1)>0$.
If $\omega_r\notin2\pi\mathbb Z$, then $\cos(\omega_r)<1$,
and hence we have
\begin{align*}
    1-\rho_{mbr}^{\mathrm{round}}
    &=
    \sum_{k\in\mathbb Z}
    \Pr(D_{mb}=k)\left[1-\cos(\omega_rk)\right]\\
    &\geq
    \Pr(D_{mb}=1)\left[1-\cos(\omega_r)\right]\\
    &>0.
\end{align*}
Similarly,
\begin{align*}
    1+\rho_{mbr}^{\mathrm{round}}
    &=
    \sum_{k\in\mathbb Z}
    \Pr(D_{mb}=k)\left[1+\cos(\omega_rk)\right]\\
    &\geq2\Pr(D_{mb}=0)\\
    &>0.
\end{align*}
Thus $-1<\rho_{mbr}^{\mathrm{round}}<1$, proving
$|\rho_{mbr}^{\mathrm{round}}|<1$.

\emph{Continuous relaxation case.}
Let $G\sim\mathcal N(0,\sigma^2)$ and write its characteristic
function as
$\varphi(\omega)=\mathbb E e^{\mathrm i\omega G}$.
For $\sigma>0$, differentiation under the expectation is
justified by $\mathbb E|G|<\infty$.
Let $f_\sigma$ denote the density of $\mathcal N(0,\sigma^2)$.
Since $gf_\sigma(g)=-\sigma^2f_\sigma'(g)$,
integration by parts gives
\begin{align*}
    \varphi'(\omega)
    &=
    \int_{\mathbb R}
    \mathrm i g e^{\mathrm i\omega g}f_\sigma(g)\,dg\\
    &=
    -\mathrm i\sigma^2
    \int_{\mathbb R}
    e^{\mathrm i\omega g}f_\sigma'(g)\,dg\\
    &=
    -\sigma^2\omega\varphi(\omega),
\end{align*}
where the boundary term vanishes because the Gaussian density
tends to zero at infinity.
Together with $\varphi(0)=1$, this differential equation yields
\begin{align*}
    \varphi(\omega)
    =
    \exp\left(-\frac{\sigma^2\omega^2}{2}\right).
\end{align*}
Therefore,
\begin{align*}
    \rho_{mbr}^{\mathrm{cont}}
    =
    \operatorname{Re}\varphi(\omega_r)
    =
    \exp\left(-\frac{\sigma^2\omega_r^2}{2}\right).
\end{align*}
For $\omega_r\neq0$ and $\sigma>0$,
\begin{align*}
    \frac{\partial}{\partial\sigma}
    \rho_{mbr}^{\mathrm{cont}}
    =
    -\sigma\omega_r^2
    \exp\left(-\frac{\sigma^2\omega_r^2}{2}\right)
    <0,
\end{align*}
which proves the stated monotonicity.
When $\sigma=0$, both the rounded and continuous offset
distributions concentrate at zero, giving $\rho_{mbr}=1$. This completes our proof.
\end{proof}

This proposition characterizes the direct key rotation for
fixed pre-RoPE states; changes to intermediate representations
under the full perturbed forward pass are outside its scope.
The bound $|\rho_{mbr}|\leq1$ concerns each expected
RoPE-frequency component and does not imply that the softmax
attention weight scales linearly by $\rho_{mbr}$.
The perturbation directly modifies only key rotations,
leaving query rotations, value vectors, and the causal mask
unchanged.

\subsection{Exact Loss Decomposition and Editing Gradients}
\label{app:consistency}

Here
we connect the proposed positional perturbation to the token-level learning signal. 
Fix a target token $c_t$ in sentence $m$ and omit the token index $t$ for brevity. Let
$\boldsymbol D_m=(D_{mb})_{b\in\mathcal B_m}$ denote the
sentence's vector of context-block offsets, as defined in
the preceding subsection.

For DO, let $z_{m,v}(\boldsymbol D_m;\theta)$ denote the
pre-softmax logit for vocabulary token $v$ at the prediction
step for $c_t$. These logits are produced by the full
perturbed forward pass, including the changes to intermediate
representations induced by the offsets.
The native and perturbed correct-versus-rest margins are
\begin{align*}
    m_{m,0}(\theta)
    &\triangleq
    z_{m,c_t}(\mathbf 0;\theta)
    -
    \log\sum_{v\neq c_t}
    \exp\bigl(z_{m,v}(\mathbf 0;\theta)\bigr),\\
    m_{m,\boldsymbol D_m}(\theta)
    &\triangleq
    z_{m,c_t}(\boldsymbol D_m;\theta)
    -
    \log\sum_{v\neq c_t}
    \exp\bigl(z_{m,v}(\boldsymbol D_m;\theta)\bigr).
\end{align*}
Both margins use the same current parameters $\theta$;
the native margin is evaluated with zero positional offsets. 
Then for a margin value $u$, the token cross-entropy and its
negative derivative are
\begin{align*}
    \ell(u)
    &=\log(1+\exp(-u)),\\
    C(u)
    &\triangleq-\ell'(u)
    =\frac{1}{1+\exp(u)}.
\end{align*}

\paragraph{Exact loss decomposition.}
Define
\begin{align*}
    \mu_m(\theta)
    &\triangleq
    \mathbb E_{\boldsymbol D_m}
    \left[m_{m,\boldsymbol D_m}(\theta)\right],\\
    p_{m,\boldsymbol D_m}(\theta)
    &\triangleq
    \operatorname{sigmoid}
    \left(m_{m,\boldsymbol D_m}(\theta)\right),\\
    p_{m,\mu}(\theta)
    &\triangleq
    \operatorname{sigmoid}\left(\mu_m(\theta)\right).
\end{align*}
Here $p_{m,\boldsymbol D_m}$ is the softmax probability
assigned to the correct token, and $p_{m,\mu}$ is the
probability corresponding to the mean margin.
In general,
$p_{m,\mu}\neq
\mathbb E_{\boldsymbol D_m}[p_{m,\boldsymbol D_m}]$.
We suppress the dependence on $\theta$ when it is clear
from context.
Then we have the following property.

\begin{proposition}[Mean-margin and view-consistency decomposition]
\label{prop:app:consistency}
Suppose that
$\mathbb E_{\boldsymbol D_m}
|m_{m,\boldsymbol D_m}|<\infty$.
Define
$\Delta_m\triangleq m_{m,0}-\mu_m$,
the token-specific quantity denoted by $\Delta$ in
Theorem~\ref{thm:krp:informal}.
Then
\begin{align}
    \mathbb E_{\boldsymbol D_m}
    \left[\ell(m_{m,\boldsymbol D_m})\right]
    &=
    \ell(\mu_m)+\mathcal R_m
    =
    \ell(m_{m,0}-\Delta_m)+\mathcal R_m,
    \label{eq:app:consistency}\\
    \mathcal R_m
    &\triangleq
    \mathbb E_{\boldsymbol D_m}
    \left[
        \operatorname{KL}\left(
            \operatorname{Bern}(p_{m,\mu})
            \,\|\,
            \operatorname{Bern}(p_{m,\boldsymbol D_m})
        \right)
    \right]
    \geq0.
    \notag
\end{align}
Moreover, $\mathcal R_m=0$ if and only if
$p_{m,\boldsymbol D_m}=p_{m,\mu}$ almost surely.
\end{proposition}

\begin{proof}
All expectations in this proof are taken with respect to
$\boldsymbol D_m$. Define
\begin{align*}
    F(u)\triangleq\log(1+\exp(u)).
\end{align*}
The integrability assumption ensures that the expectations
below are finite, since
\begin{align*}
    F(u)\leq\log 2+|u|,
    \qquad
    \ell(u)\leq\log 2+|u|.
\end{align*}
For $p(u)=\operatorname{sigmoid}(u)$, we have
\begin{align*}
    \log p(u)&=u-F(u),\\
    \log(1-p(u))&=-F(u).
\end{align*}
Applying these identities at
$u=m_{m,\boldsymbol D_m}$ and $u=\mu_m$ gives
\begin{align*}
    &\operatorname{KL}\left(
        \operatorname{Bern}(p_{m,\mu})
        \,\|\,
        \operatorname{Bern}(p_{m,\boldsymbol D_m})
    \right)\\
    &\quad=
    p_{m,\mu}
    \log\frac{p_{m,\mu}}{p_{m,\boldsymbol D_m}}
    +
    (1-p_{m,\mu})
    \log\frac{1-p_{m,\mu}}{1-p_{m,\boldsymbol D_m}}\\
    &\quad=
    p_{m,\mu}
    \left[
        \mu_m-m_{m,\boldsymbol D_m}
        -F(\mu_m)+F(m_{m,\boldsymbol D_m})
    \right]\\
    &\qquad\quad+
    (1-p_{m,\mu})
    \left[
        F(m_{m,\boldsymbol D_m})-F(\mu_m)
    \right]\\
    &\quad=
    F(m_{m,\boldsymbol D_m})-F(\mu_m)
    -F'(\mu_m)(m_{m,\boldsymbol D_m}-\mu_m),
\end{align*}
where $F'(\mu_m)=p_{m,\mu}$.

Taking expectations and using
$\mathbb E[m_{m,\boldsymbol D_m}-\mu_m]=0$, we obtain
\begin{align*}
    \mathcal R_m
    &=
    \mathbb E F(m_{m,\boldsymbol D_m})-F(\mu_m)\\
    &=
    \mathbb E\left[
        \ell(m_{m,\boldsymbol D_m})+m_{m,\boldsymbol D_m}
    \right]
    -
    \left[\ell(\mu_m)+\mu_m\right]\\
    &=
    \mathbb E\ell(m_{m,\boldsymbol D_m})-\ell(\mu_m),
\end{align*}
where the second equality uses $\ell(u)=F(u)-u$.
Rearranging and substituting
$\mu_m=m_{m,0}-\Delta_m$ proves the decomposition.

The same identity shows that $\mathcal R_m$ is the Jensen
gap of the loss:
\begin{align*}
    \mathcal R_m
    =
    \mathbb E\ell(m_{m,\boldsymbol D_m})
    -
    \ell\left(\mathbb E m_{m,\boldsymbol D_m}\right)
    \geq0.
\end{align*}
Since
\begin{align*}
    \ell''(u)=C(u)(1-C(u))>0
    \qquad\text{for all }u\in\mathbb R,
\end{align*}
the loss is strictly convex. Equality therefore holds if
and only if $m_{m,\boldsymbol D_m}=\mu_m$ almost surely.
The sigmoid function is injective, so this is equivalent to
$p_{m,\boldsymbol D_m}=p_{m,\mu}$ almost surely.
\end{proof}

The decomposition provides an analytical interpretation of
the perturbed NLL objective: $\Delta_m$ records the reduction
in mean margin, and $\mathcal R_m$ measures variation in
correct-token confidence across positional views.
This connects the objective to the noise-as-regularization
viewpoint~\citep{bishop1995training,wager2013dropout}.

\paragraph{Editing gradients.}
Now we differentiate the loss decomposition to characterize
how positional perturbations change the NLL gradients
for DO parameters and LTE Stage 1 residuals.

\begin{corollary}[Gradients for parameters and hidden targets]
\label{cor:app:gradient}
Under the integrability condition of
Proposition~\ref{prop:app:consistency}, suppose that the
offset distribution is independent of $\theta$ and that the
native and perturbed margins are differentiable in $\theta$.
Assume that, throughout a neighborhood of $\theta$,
\begin{align*}
    \|\nabla_\theta m_{m,\boldsymbol D_m}\|_2
    \leq H_m(\boldsymbol D_m),
    \qquad
    \mathbb E_{\boldsymbol D_m}H_m(\boldsymbol D_m)<\infty.
\end{align*}
Define the expected token loss
\begin{align*}
    \mathcal J_m(\theta)
    \triangleq
    \mathbb E_{\boldsymbol D_m}
    \left[\ell(m_{m,\boldsymbol D_m}(\theta))\right].
\end{align*}
Then
\begin{align}
    \nabla_\theta\mathcal J_m
    =
    -C(\mu_m)\nabla_\theta\mu_m
    +\nabla_\theta\mathcal R_m.
    \label{eq:app:editing-gradient}
\end{align}
Relative to the native token loss,
\begin{align}
    \nabla_\theta\mathcal J_m
    -\nabla_\theta\ell(m_{m,0})
    ={}&
    -[C(\mu_m)-C(m_{m,0})]\nabla_\theta m_{m,0}
    \notag\\
    &+
    C(\mu_m)\nabla_\theta\Delta_m
    +\nabla_\theta\mathcal R_m.
    \label{eq:app:gradient-difference}
\end{align}

Write
\begin{align*}
    g_{m,\boldsymbol D_m}
    &\triangleq\nabla_\theta m_{m,\boldsymbol D_m},
    \qquad
    \bar g_m\triangleq
    \mathbb E_{\boldsymbol D_m}g_{m,\boldsymbol D_m}.
\end{align*}
If $m_{m,\boldsymbol D_m}$ and $g_{m,\boldsymbol D_m}$
have finite second moments, define
\begin{align*}
    V_m^{\mathrm{margin}}
    &\triangleq
    \operatorname{Var}(m_{m,\boldsymbol D_m}),\\
    V_m^{\mathrm{grad}}
    &\triangleq
    \mathbb E_{\boldsymbol D_m}
    \left\|g_{m,\boldsymbol D_m}-\bar g_m\right\|_2^2.
\end{align*}
Then
\begin{align}
    \|\nabla_\theta\mathcal R_m\|_2
    \leq
    \frac{V_m^{\mathrm{margin}}}{8}\|\bar g_m\|_2
    +
    \frac14
    \sqrt{V_m^{\mathrm{margin}}V_m^{\mathrm{grad}}}.
    \label{eq:app:consistency-gradient-bound}
\end{align}

For LTE Stage 1, the same identities hold with
$\nabla_\theta$ replaced by $\nabla_{\delta_m}$,
where $\delta_m$ is the sentence-specific residual in
Eq.~\eqref{eq:krp:delta} and the model parameters
$\theta_0$ are fixed.
The native and perturbed margins use the same carrier
substitution and residual injection, with zero offsets
in the native computation.
\end{corollary}

\begin{proof}
All expectations are over $\boldsymbol D_m$.
We first differentiate the decomposition, then bound
the gradient of the consistency term.

\emph{Differentiating the expectation.}
The dominating bound and $|\ell'(u)|\leq1$ give
\begin{align*}
    \left\|
        \ell'(m_{m,\boldsymbol D_m})
        \nabla_\theta m_{m,\boldsymbol D_m}
    \right\|_2
    \leq H_m(\boldsymbol D_m).
\end{align*}
Differentiation under the expectation is therefore justified,
yielding
\begin{align*}
    \nabla_\theta\mu_m
    &=
    \mathbb E
    \left[\nabla_\theta m_{m,\boldsymbol D_m}\right]
    =\bar g_m,\\
    \nabla_\theta\mathcal J_m
    &=
    \mathbb E
    \left[
        \ell'(m_{m,\boldsymbol D_m})
        g_{m,\boldsymbol D_m}
    \right]\\
    &=
    -\mathbb E
    \left[
        C(m_{m,\boldsymbol D_m})g_{m,\boldsymbol D_m}
    \right],\\
    \nabla_\theta\ell(\mu_m)
    &=-C(\mu_m)\bar g_m.
\end{align*}
Using $\mathcal R_m=\mathcal J_m-\ell(\mu_m)$, we obtain
\begin{align}
    \nabla_\theta\mathcal R_m
    =
    C(\mu_m)\bar g_m
    -
    \mathbb E
    \left[
        C(m_{m,\boldsymbol D_m})g_{m,\boldsymbol D_m}
    \right],
    \label{eq:app:consistency-gradient}
\end{align}
which proves Eq.~\eqref{eq:app:editing-gradient}.

Next, substitute
$\nabla_\theta\mu_m
=\nabla_\theta m_{m,0}-\nabla_\theta\Delta_m$
and subtract the native gradient:
\begin{align*}
    &\nabla_\theta\mathcal J_m
    -\nabla_\theta\ell(m_{m,0})\\
    &\quad=
    -C(\mu_m)
    \left(
        \nabla_\theta m_{m,0}
        -\nabla_\theta\Delta_m
    \right)
    +\nabla_\theta\mathcal R_m
    +C(m_{m,0})\nabla_\theta m_{m,0}\\
    &\quad=
    -[C(\mu_m)-C(m_{m,0})]\nabla_\theta m_{m,0}
    +C(\mu_m)\nabla_\theta\Delta_m
    +\nabla_\theta\mathcal R_m.
\end{align*}
This proves Eq.~\eqref{eq:app:gradient-difference}.

\emph{Mean and covariance contributions.}
Let
$\bar C_m=\mathbb E[C(m_{m,\boldsymbol D_m})]$.
Centering the product in
Eq.~\eqref{eq:app:consistency-gradient} gives
\begin{align}
    \nabla_\theta\mathcal R_m
    ={}&
    [C(\mu_m)-\bar C_m]\bar g_m
    \notag\\
    &-
    \mathbb E\left[
        (C(m_{m,\boldsymbol D_m})-\bar C_m)
        (g_{m,\boldsymbol D_m}-\bar g_m)
    \right].
    \label{eq:app:gradient-covariance}
\end{align}
We bound the two terms separately.
Since
$\sup_u|C'(u)|\leq1/4$ and
$\sup_u|C''(u)|\leq1/4$,
Taylor's formula around $\mu_m$ gives
\begin{align*}
    |\bar C_m-C(\mu_m)|
    &=
    \left|
        \mathbb E\left[
            C(m_{m,\boldsymbol D_m})-C(\mu_m)
            -C'(\mu_m)(m_{m,\boldsymbol D_m}-\mu_m)
        \right]
    \right|\\
    &\leq
    \frac18
    \mathbb E(m_{m,\boldsymbol D_m}-\mu_m)^2\\
    &=
    \frac{V_m^{\mathrm{margin}}}{8}.
\end{align*}
The derivative bound also yields
\begin{align*}
    \operatorname{Var}(C(m_{m,\boldsymbol D_m}))
    &\leq
    \mathbb E\left[
        (C(m_{m,\boldsymbol D_m})-C(\mu_m))^2
    \right]\\
    &\leq
    \frac{V_m^{\mathrm{margin}}}{16}.
\end{align*}
Applying vector Cauchy--Schwarz to the covariance term,
\begin{align*}
    &\left\|
        \mathbb E\left[
            (C(m_{m,\boldsymbol D_m})-\bar C_m)
            (g_{m,\boldsymbol D_m}-\bar g_m)
        \right]
    \right\|_2\\
    &\quad\leq
    \sqrt{
        \operatorname{Var}(C(m_{m,\boldsymbol D_m}))
        \,
        \mathbb E\|g_{m,\boldsymbol D_m}-\bar g_m\|_2^2
    }\\
    &\quad\leq
    \frac14
    \sqrt{V_m^{\mathrm{margin}}V_m^{\mathrm{grad}}}.
\end{align*}
Combining these bounds with the triangle inequality proves
Eq.~\eqref{eq:app:consistency-gradient-bound}.
The same argument applies to differentiation with respect
to $\delta_m$ under the corresponding assumptions.
This completes our proof.
\end{proof}

Averaging these token-level identities within each sentence
gives the corresponding LTE Stage 1 NLL identities.
For DO, further aggregating sentence losses with weights
$|\mathcal I_m|/|\cv|$ gives the perturbed NLL objective
in Eq.~\eqref{eq:do}.
When $\Delta_m\geq0$, the first term in
Eq.~\eqref{eq:app:gradient-difference} reinforces the native
margin correction, while the overall gradient change
depends on all three terms.

\paragraph{Local sensitivity to block offsets.}
The consistency term admits a local interpretation in
terms of margin variance.
Let
\begin{align*}
    h_{m,\boldsymbol D_m}
    \triangleq m_{m,\boldsymbol D_m}-\mu_m,
    \qquad
    \mathbb E h_{m,\boldsymbol D_m}=0.
\end{align*}
All expectations below are over $\boldsymbol D_m$ or,
in the continuous construction, its underlying Gaussian
vector.

Taylor's formula with an integral remainder gives
\begin{align*}
    &\ell(\mu_m+h_{m,\boldsymbol D_m})
    -\ell(\mu_m)
    -\ell'(\mu_m)h_{m,\boldsymbol D_m}\\
    &\quad=
    h_{m,\boldsymbol D_m}^2
    \int_0^1
    (1-u)\ell''(\mu_m+u h_{m,\boldsymbol D_m})\,du.
\end{align*}
Taking expectations yields
\begin{align*}
    \mathcal R_m
    =
    \mathbb E\left[
        h_{m,\boldsymbol D_m}^2
        \int_0^1
        (1-u)\ell''(\mu_m+u h_{m,\boldsymbol D_m})\,du
    \right].
\end{align*}
If the margin has a finite second moment, then
$0<\ell''\leq1/4$ implies
\begin{align*}
    0\leq\mathcal R_m
    \leq
    \frac18\operatorname{Var}(m_{m,\boldsymbol D_m}).
\end{align*}

When the third centered absolute moment is finite,
a second-order expansion gives
\begin{align*}
    \ell(\mu_m+h_{m,\boldsymbol D_m})
    ={}&
    \ell(\mu_m)
    +\ell'(\mu_m)h_{m,\boldsymbol D_m}\\
    &+
    \frac{\ell''(\mu_m)}{2}h_{m,\boldsymbol D_m}^2
    +q_{m,\boldsymbol D_m},
\end{align*}
where
\begin{align*}
    q_{m,\boldsymbol D_m}
    &=
    \frac{h_{m,\boldsymbol D_m}^3}{2}
    \int_0^1
    (1-u)^2
    \ell'''(\mu_m+u h_{m,\boldsymbol D_m})\,du,\\
    |q_{m,\boldsymbol D_m}|
    &\leq
    \frac{|h_{m,\boldsymbol D_m}|^3}{24}.
\end{align*}
The bound follows from
$|\ell'''|=|C(1-C)(2C-1)|\leq1/4$.
Taking expectations and writing
$r_{m,3}=\mathbb E q_{m,\boldsymbol D_m}$ gives
\begin{align}
    \mathcal R_m
    &=
    \frac{\ell''(\mu_m)}{2}
    \operatorname{Var}(m_{m,\boldsymbol D_m})
    +r_{m,3},
    \label{eq:app:consistency-variance}\\
    |r_{m,3}|
    &\leq
    \frac1{24}
    \mathbb E|m_{m,\boldsymbol D_m}-\mu_m|^3.
    \notag
\end{align}

We now specialize to a small continuous perturbation
\begin{align*}
    D_{mb}=\varepsilon\xi_{mb},
    \qquad
    \xi_{mb}\sim\mathcal N(0,\sigma^2),
    \qquad b\in\mathcal B_m,
\end{align*}
with independent draws across blocks.
Let
$\boldsymbol\xi_m=(\xi_{mb})_{b\in\mathcal B_m}$, and define
the sensitivity to each shared block offset by
\begin{align*}
    J_{mb}
    \triangleq
    \left.
    \frac{\partial m_{m,\boldsymbol D_m}}{\partial D_{mb}}
    \right|_{\boldsymbol D_m=\mathbf 0}.
\end{align*}
At the fixed current model state, suppose that the margin
admits the expansion
\begin{align*}
    m_{m,\varepsilon\boldsymbol\xi_m}
    &=
    m_{m,0}
    +\varepsilon Y_m
    +\varepsilon^2 e_{m,\varepsilon},
    \qquad
    Y_m=\sum_{b\in\mathcal B_m}J_{mb}\xi_{mb},
\end{align*}
where, uniformly for sufficiently small $\varepsilon>0$,
\begin{align*}
    |e_{m,\varepsilon}|
    \leq K_m(\boldsymbol\xi_m),
    \qquad
    \mathbb E K_m(\boldsymbol\xi_m)^3<\infty.
\end{align*}
For this family of perturbations, $\mu_m$ and $\mathcal R_m$
are evaluated at
$\boldsymbol D_m=\varepsilon\boldsymbol\xi_m$.

Independence across blocks gives
\begin{align*}
    \mathbb E Y_m&=0,\\
    \mathbb E Y_m^2
    &=
    \sum_{b,b'\in\mathcal B_m}
    J_{mb}J_{mb'}
    \mathbb E[\xi_{mb}\xi_{mb'}]\\
    &=
    \sigma^2\sum_{b\in\mathcal B_m}J_{mb}^2,\\
    \mu_m
    &=
    m_{m,0}+\varepsilon^2\mathbb E e_{m,\varepsilon},\\
    h_{m,\varepsilon\boldsymbol\xi_m}
    &=
    \varepsilon Y_m
    +\varepsilon^2
    (e_{m,\varepsilon}-\mathbb E e_{m,\varepsilon}).
\end{align*}
Expanding the variance, we obtain
\begin{align*}
    \operatorname{Var}(m_{m,\varepsilon\boldsymbol\xi_m})
    ={}&
    \varepsilon^2\mathbb E Y_m^2
    +2\varepsilon^3
    \operatorname{Cov}(Y_m,e_{m,\varepsilon})\\
    &+
    \varepsilon^4\operatorname{Var}(e_{m,\varepsilon}).
\end{align*}
Writing $K_m=K_m(\boldsymbol\xi_m)$ for brevity,
by
Cauchy--Schwarz inequality, we have
\begin{align*}
    &\left|
        \operatorname{Var}(m_{m,\varepsilon\boldsymbol\xi_m})
        -\varepsilon^2\mathbb E Y_m^2
    \right|\\
    &\quad\leq
    2\varepsilon^3
    \sqrt{\mathbb E Y_m^2\,\mathbb E K_m^2}
    +\varepsilon^4\mathbb E K_m^2
    =O(\varepsilon^3).
\end{align*}
Similarly,
\begin{align*}
    \mathbb E|h_{m,\varepsilon\boldsymbol\xi_m}|^3
    \leq{}&
    4\varepsilon^3\mathbb E|Y_m|^3\\
    &+
    4\varepsilon^6
    \mathbb E
    |e_{m,\varepsilon}-\mathbb E e_{m,\varepsilon}|^3
    =O(\varepsilon^3).
\end{align*}
Finally,
\begin{align*}
    |\ell''(\mu_m)-\ell''(m_{m,0})|
    \leq\frac14|\mu_m-m_{m,0}|
    =O(\varepsilon^2).
\end{align*}
Substituting these estimates into
Eq.~\eqref{eq:app:consistency-variance} gives
\begin{align}
    \mathcal R_m
    =
    \frac{\varepsilon^2\sigma^2\ell''(m_{m,0})}{2}
    \sum_{b\in\mathcal B_m}J_{mb}^2
    +O(\varepsilon^3).
    \label{eq:app:phase-sensitivity}
\end{align}

Each derivative $J_{mb}$ accounts for the shared offset's
effect through all eligible keys and attention layers.
Thus, to leading order in this continuous local
approximation, the consistency term is proportional to
the squared sensitivity of the target-token margin to
the context-block offsets.
For the finite rounded offsets used by \name, the exact
decomposition is given by Eq.~\eqref{eq:app:consistency}.

\subsection{Sign of the Mean-Margin Term under a Local Shortcut Model}
\label{app:margin}

We now derive sufficient conditions for the mean-margin
reduction $\Delta_m=m_{m,0}-\mu_m$ to be nonnegative.
Following Appendix~\ref{app:consistency}, we fix a target
token in sentence $m$ and suppress its token index.
We connect the spectral contraction in
Proposition~\ref{prop:app:contraction} to the margin
through an affine model of the RoPE rotation components.

\begin{assumption}[Positive contextual shortcut]
\label{ass:app:shortcut}
At the current optimization state, the margin is modeled as
\begin{align*}
    m_{m,\boldsymbol D_m}
    =
    u_m
    +
    \sum_{b\in\mathcal B_m}
    \sum_{r=1}^{R}
    \left[
        A_{mbr}\cos(\omega_rD_{mb})
        +
        B_{mbr}\sin(\omega_rD_{mb})
    \right],
\end{align*}
where $u_m$, $A_{mbr}$, and $B_{mbr}$ are finite real
coefficients held fixed across positional views, and
\begin{align*}
    A_{mbr}\geq0
    \qquad
    \text{for every }b\in\mathcal B_m
    \text{ and }r=1,\ldots,R.
\end{align*}
Here $u_m$ is the offset-independent contribution.
At zero offsets, the model gives
\begin{align*}
    m_{m,0}
    =
    u_m
    +
    \sum_{b\in\mathcal B_m}
    \sum_{r=1}^{R}A_{mbr}.
\end{align*}
\end{assumption}

The cosine and sine terms arise by expanding the additional
rotation $\operatorname{Rot}(\omega_rD_{mb})$ in an affine
function of the rotation components.
The positivity condition applies to each native
block-frequency contribution $A_{mbr}$ individually.

\begin{theorem}[Mean-margin decrease under a positive shortcut]
\label{thm:app:correction}
Under the affine margin model in
Assumption~\ref{ass:app:shortcut}, suppose that each block
offset $D_{mb}$ has a distribution symmetric about zero.
Let
\begin{align*}
    \rho_{mbr}
    =
    \mathbb E_{D_{mb}}
    \left[\cos(\omega_rD_{mb})\right]
\end{align*}
be the spectral multiplier from
Proposition~\ref{prop:app:contraction}.
Then the mean-margin reduction defined in
Proposition~\ref{prop:app:consistency} satisfies
\begin{align*}
    \Delta_m
    =
    m_{m,0}-\mu_m
    =
    \sum_{b\in\mathcal B_m}
    \sum_{r=1}^{R}
    (1-\rho_{mbr})A_{mbr}
    \geq0.
\end{align*}
Consequently,
\begin{align*}
    \mathbb E_{\boldsymbol D_m}
    \left[\ell(m_{m,\boldsymbol D_m})\right]
    &=
    \ell(m_{m,0}-\Delta_m)+\mathcal R_m\\
    &\geq
    \ell(m_{m,0}-\Delta_m)\\
    &\geq
    \ell(m_{m,0}).
\end{align*}
For fixed native margin $m_{m,0}$, the mean-margin loss
$\ell(m_{m,0}-\Delta_m)$ increases strictly with $\Delta_m$:
\begin{align*}
    \frac{\partial}{\partial\Delta_m}
    \ell(m_{m,0}-\Delta_m)
    =
    C(m_{m,0}-\Delta_m)>0.
\end{align*}
\end{theorem}

\begin{proof}
Under Assumption~\ref{ass:app:shortcut}, the margin is bounded
in absolute value by
\begin{align*}
    |m_{m,\boldsymbol D_m}|
    \leq
    |u_m|
    +
    \sum_{b\in\mathcal B_m}
    \sum_{r=1}^{R}
    \left(|A_{mbr}|+|B_{mbr}|\right),
\end{align*}
so the integrability condition of
Proposition~\ref{prop:app:consistency} holds.

By symmetry of each block-offset distribution,
\begin{align*}
    \mathbb E_{D_{mb}}\sin(\omega_rD_{mb})=0,
    \qquad
    \mathbb E_{D_{mb}}\cos(\omega_rD_{mb})=\rho_{mbr}.
\end{align*}
Taking expectations in the affine model therefore gives
\begin{align*}
    \mu_m
    &=
    \mathbb E_{\boldsymbol D_m}
    \left[m_{m,\boldsymbol D_m}\right]\\
    &=
    u_m
    +
    \sum_{b\in\mathcal B_m}
    \sum_{r=1}^{R}
    \rho_{mbr}A_{mbr}\\
    &=
    \left(
        u_m+
        \sum_{b\in\mathcal B_m}
        \sum_{r=1}^{R}A_{mbr}
    \right)
    -
    \sum_{b\in\mathcal B_m}
    \sum_{r=1}^{R}
    (1-\rho_{mbr})A_{mbr}\\
    &=
    m_{m,0}
    -
    \sum_{b\in\mathcal B_m}
    \sum_{r=1}^{R}
    (1-\rho_{mbr})A_{mbr}.
\end{align*}
Thus,
\begin{align*}
    \Delta_m
    =
    m_{m,0}-\mu_m
    =
    \sum_{b\in\mathcal B_m}
    \sum_{r=1}^{R}
    (1-\rho_{mbr})A_{mbr}.
\end{align*}
Proposition~\ref{prop:app:contraction} gives
$|\rho_{mbr}|\leq1$, hence $1-\rho_{mbr}\geq0$.
Together with $A_{mbr}\geq0$, this proves $\Delta_m\geq0$.

Substituting $\mu_m=m_{m,0}-\Delta_m$ into
Eq.~\eqref{eq:app:consistency} yields
\begin{align*}
    \mathbb E_{\boldsymbol D_m}
    \left[\ell(m_{m,\boldsymbol D_m})\right]
    =
    \ell(m_{m,0}-\Delta_m)+\mathcal R_m.
\end{align*}
Since $\mathcal R_m\geq0$, the expected loss is at least
$\ell(m_{m,0}-\Delta_m)$.
Moreover, $\ell'(u)=-C(u)<0$ and $\Delta_m\geq0$ imply
\begin{align*}
    \ell(m_{m,0}-\Delta_m)\geq\ell(m_{m,0}).
\end{align*}
Finally, holding $m_{m,0}$ fixed and differentiating gives
\begin{align*}
    \frac{\partial}{\partial\Delta_m}
    \ell(m_{m,0}-\Delta_m)
    =
    -\ell'(m_{m,0}-\Delta_m)
    =
    C(m_{m,0}-\Delta_m)>0,
\end{align*}
which proves the stated monotonicity.
\end{proof}

\subsection{Beyond the Affine Margin Model}
\label{app:curvature}

We now bound the error of the affine approximation by
expressing the full margin computation as a function of
the RoPE rotation components.
As in the preceding subsections, we fix a target token in
sentence $m$, retain the sentence index, and suppress the
token index. The current optimization variables are held
fixed throughout this analysis.

For each context block $b\in\mathcal B_m$ and frequency $r$,
define
\begin{align*}
    \boldsymbol\chi_{mbr}(\boldsymbol D_m)
    &=
    \begin{pmatrix}
        \cos(\omega_rD_{mb})\\
        \sin(\omega_rD_{mb})
    \end{pmatrix},
    \qquad
    \boldsymbol\chi_{mbr}^{(0)}
    =
    \begin{pmatrix}1\\0\end{pmatrix}.
\end{align*}
Let $\boldsymbol\chi_m(\boldsymbol D_m)$ collect these
coordinate pairs over all blocks and frequencies, and
let $\boldsymbol\chi_m^{(0)}$ collect their native values.
To define the margin as a function of these coordinates,
write the additional key-side transformation as
\begin{align*}
    \mathbf T_{mbr}(\boldsymbol\chi_{mbr})
    =
    \begin{pmatrix}
        \chi_{mbr,1}&-\chi_{mbr,2}\\
        \chi_{mbr,2}&\chi_{mbr,1}
    \end{pmatrix}.
\end{align*}
At the sampled coordinates, this matrix equals
$\operatorname{Rot}(\omega_rD_{mb})$.
Each transformation is applied to the eligible keys and
attention layers sharing block $b$ and frequency $r$.
We treat $\chi_{mbr,1}$ and $\chi_{mbr,2}$ as separate
coordinates when differentiating the margin.
Let $f_m$ denote the resulting full margin computation,
including the changes to intermediate representations.
Then
\begin{align*}
    m_{m,\boldsymbol D_m}
    &=
    f_m\bigl(\boldsymbol\chi_m(\boldsymbol D_m)\bigr),\\
    m_{m,0}
    &=
    f_m\bigl(\boldsymbol\chi_m^{(0)}\bigr).
\end{align*}

\begin{proposition}[Mean-margin decrease beyond the affine model]
\label{prop:app:curvature}
Suppose that $f_m$ is twice continuously differentiable
on a neighborhood of every segment joining
$\boldsymbol\chi_m^{(0)}$ to a possible
$\boldsymbol\chi_m(\boldsymbol D_m)$, with
\begin{align*}
    \left\|
        \nabla_{\boldsymbol\chi_m}^2 f_m
    \right\|_{\mathrm{op}}
    \leq\Lambda_m
\end{align*}
throughout these segments, where $\Lambda_m$ is a finite
nonnegative constant.
Assume that each block offset has a distribution symmetric
about zero, and let
\begin{align*}
    \rho_{mbr}
    =
    \mathbb E_{D_{mb}}
    \left[\cos(\omega_rD_{mb})\right].
\end{align*}
Define
\begin{align*}
    S_{mbr}
    &\triangleq
    \left.
    \frac{\partial f_m}{\partial\chi_{mbr,1}}
    \right|_{\boldsymbol\chi_m=\boldsymbol\chi_m^{(0)}},\\
    \Delta_m^{\mathrm{lin}}
    &\triangleq
    \sum_{b\in\mathcal B_m}
    \sum_{r=1}^{R}
    (1-\rho_{mbr})S_{mbr},\\
    \mathcal E_m
    &\triangleq
    \Lambda_m
    \sum_{b\in\mathcal B_m}
    \sum_{r=1}^{R}
    (1-\rho_{mbr}).
\end{align*}
Then the actual mean-margin reduction
$\Delta_m=m_{m,0}-\mu_m$ satisfies
\begin{align}
    \left|\Delta_m-\Delta_m^{\mathrm{lin}}\right|
    \leq\mathcal E_m,
    \label{eq:app:margin-remainder}
\end{align}
and the expected loss satisfies
\begin{align}
    \mathbb E_{\boldsymbol D_m}
    \left[\ell(m_{m,\boldsymbol D_m})\right]
    \geq
    \ell\left(
        m_{m,0}-\Delta_m^{\mathrm{lin}}+\mathcal E_m
    \right).
    \label{eq:app:curvature-loss}
\end{align}
In particular,
$\Delta_m^{\mathrm{lin}}>\mathcal E_m$ implies
$\Delta_m>0$ and
\begin{align*}
    \mathbb E_{\boldsymbol D_m}
    \left[\ell(m_{m,\boldsymbol D_m})\right]
    >
    \ell(m_{m,0}).
\end{align*}
\end{proposition}

\begin{proof}
All expectations are taken with respect to
$\boldsymbol D_m$.
Write
\begin{align*}
    \mathbf e_m
    \triangleq
    \boldsymbol\chi_m(\boldsymbol D_m)
    -
    \boldsymbol\chi_m^{(0)}
\end{align*}
for the random displacement in the rotation coordinates.
Taylor's theorem along the segment
$\boldsymbol\chi_m^{(0)}+u\mathbf e_m$,
$0\leq u\leq1$, gives
\begin{align*}
    m_{m,\boldsymbol D_m}
    =
    m_{m,0}
    +
    \nabla f_m(\boldsymbol\chi_m^{(0)})^\top\mathbf e_m
    +r_m,
\end{align*}
where
\begin{align*}
    r_m
    =
    \int_0^1
    (1-u)\,
    \mathbf e_m^\top
    \nabla^2 f_m
    \left(\boldsymbol\chi_m^{(0)}+u\mathbf e_m\right)
    \mathbf e_m\,du.
\end{align*}
The Hessian bound yields
\begin{align*}
    |r_m|
    &\leq
    \int_0^1
    (1-u)
    \left\|
        \nabla^2 f_m
        \left(\boldsymbol\chi_m^{(0)}+u\mathbf e_m\right)
    \right\|_{\mathrm{op}}
    \|\mathbf e_m\|_2^2\,du\\
    &\leq
    \frac{\Lambda_m}{2}\|\mathbf e_m\|_2^2.
\end{align*}
Since the rotation coordinates are bounded, this expansion
also ensures that the margin and its loss are integrable.

For the coordinate pair associated with block $b$ and
frequency $r$, symmetry gives
\begin{align*}
    \mathbb E\mathbf e_{mbr}
    &=
    \begin{pmatrix}
        \mathbb E\cos(\omega_rD_{mb})-1\\
        \mathbb E\sin(\omega_rD_{mb})
    \end{pmatrix}\\
    &=
    \begin{pmatrix}
        \rho_{mbr}-1\\0
    \end{pmatrix}.
\end{align*}
Consequently, the expected linear term is
\begin{align*}
    \nabla f_m(\boldsymbol\chi_m^{(0)})^\top
    \mathbb E\mathbf e_m
    &=
    \sum_{b\in\mathcal B_m}
    \sum_{r=1}^{R}
    S_{mbr}(\rho_{mbr}-1)\\
    &=
    -\Delta_m^{\mathrm{lin}}.
\end{align*}

To bound the expected remainder, we compute
\begin{align*}
    \mathbb E\|\mathbf e_m\|_2^2
    &=
    \sum_{b\in\mathcal B_m}
    \sum_{r=1}^{R}
    \mathbb E\left[
        \bigl(\cos(\omega_rD_{mb})-1\bigr)^2
        +\sin^2(\omega_rD_{mb})
    \right]\\
    &=
    2\sum_{b\in\mathcal B_m}
    \sum_{r=1}^{R}
    \mathbb E\left[1-\cos(\omega_rD_{mb})\right]\\
    &=
    2\sum_{b\in\mathcal B_m}
    \sum_{r=1}^{R}(1-\rho_{mbr}).
\end{align*}
It follows that
\begin{align*}
    |\mathbb E r_m|
    &\leq\mathbb E|r_m|\\
    &\leq
    \frac{\Lambda_m}{2}\mathbb E\|\mathbf e_m\|_2^2\\
    &=
    \Lambda_m
    \sum_{b\in\mathcal B_m}
    \sum_{r=1}^{R}(1-\rho_{mbr})\\
    &=
    \mathcal E_m.
\end{align*}
Taking expectations in the Taylor expansion therefore gives
\begin{align*}
    \mu_m
    &=
    m_{m,0}-\Delta_m^{\mathrm{lin}}+\mathbb E r_m,\\
    \Delta_m
    &=
    m_{m,0}-\mu_m
    =
    \Delta_m^{\mathrm{lin}}-\mathbb E r_m.
\end{align*}
Thus,
\begin{align*}
    |\Delta_m-\Delta_m^{\mathrm{lin}}|
    =
    |\mathbb E r_m|
    \leq\mathcal E_m,
\end{align*}
which proves Eq.~\eqref{eq:app:margin-remainder}.

By convexity of $\ell$ and Jensen's inequality,
\begin{align*}
    \mathbb E\left[\ell(m_{m,\boldsymbol D_m})\right]
    &\geq\ell(\mu_m)\\
    &=
    \ell\left(
        m_{m,0}-\Delta_m^{\mathrm{lin}}+\mathbb E r_m
    \right)\\
    &\geq
    \ell\left(
        m_{m,0}-\Delta_m^{\mathrm{lin}}+\mathcal E_m
    \right),
\end{align*}
where the last inequality uses
$\mathbb E r_m\leq\mathcal E_m$ and $\ell'<0$.
This proves Eq.~\eqref{eq:app:curvature-loss}.

Finally, if $\Delta_m^{\mathrm{lin}}>\mathcal E_m$, then
\begin{align*}
    \Delta_m
    \geq
    \Delta_m^{\mathrm{lin}}-\mathcal E_m
    >0.
\end{align*}
Since $\ell$ is strictly decreasing,
\begin{align*}
    \mathbb E\left[\ell(m_{m,\boldsymbol D_m})\right]
    &\geq
    \ell\left(
        m_{m,0}-\Delta_m^{\mathrm{lin}}+\mathcal E_m
    \right)\\
    &>
    \ell(m_{m,0}).
\end{align*}
This completes the proof.
\end{proof}

\subsection{Conditional Attention-Mass Redistribution}
\label{app:attention-mass}

We now give a sufficient condition for positional
perturbations to increase the attention mass assigned to
unperturbed keys.
Fix a target sentence $m$, an attention head, and a query
position $t$ used to predict a token in that sentence.
As in Proposition~\ref{prop:app:contraction}, we hold the
pre-RoPE query and key representations fixed.

Partition the visible key positions into two disjoint sets:
$\mathcal H_m$ contains the context keys designated for
perturbation, and $\mathcal U_m$ contains the remaining
keys. The latter includes visible target-sentence keys
and the protected first-token readout key whenever visible.
We assume $\mathcal U_m$ is nonempty.

Using the logit notation from
Proposition~\ref{prop:app:contraction}, define
\begin{align*}
    Z_m^{\mathcal U}
    &\triangleq
    \sum_{u\in\mathcal U_m}e^{a_{tu}},\\
    Z_m^{\mathcal H}(\boldsymbol D_m)
    &\triangleq
    \sum_{s\in\mathcal H_m}
    e^{\widetilde a_{ts}(\boldsymbol D_m)},\\
    M_m(\boldsymbol D_m)
    &\triangleq
    \frac{Z_m^{\mathcal U}}
    {Z_m^{\mathcal U}+Z_m^{\mathcal H}(\boldsymbol D_m)}.
\end{align*}
Here $M_m(\boldsymbol D_m)$ is the total attention mass
assigned to $\mathcal U_m$, and
$M_{m,0}=M_m(\mathbf 0)$ is its native value.
Under the fixed-representation analysis,
$Z_m^{\mathcal U}$ is constant across positional views.

\begin{proposition}[Attention mass under exponential-moment contraction]
\label{prop:app:attention-mass}
Suppose that, for some $\kappa>0$,
\begin{align}
    \mathbb E_{\boldsymbol D_m}
    Z_m^{\mathcal H}(\boldsymbol D_m)
    \leq
    e^{-\kappa}Z_m^{\mathcal H}(\mathbf 0).
    \label{eq:app:partition-contraction}
\end{align}
Then
\begin{align}
    \mathbb E_{\boldsymbol D_m}
    M_m(\boldsymbol D_m)
    \geq
    \frac{1}
    {1+e^{-\kappa}(1-M_{m,0})/M_{m,0}}.
    \label{eq:app:attention-mass}
\end{align}
For any $\eta\in(0,1)$, with probability at least $1-\eta$,
\begin{align}
    M_m(\boldsymbol D_m)
    \geq
    \frac{1}
    {1+e^{-\kappa}(1-M_{m,0})/(\eta M_{m,0})}.
    \label{eq:app:attention-probability}
\end{align}
When $\mathcal H_m$ is nonempty, the expectation bound
strictly exceeds $M_{m,0}$, and the probability bound
strictly exceeds $M_{m,0}$ if
$\kappa>\log(1/\eta)$.
\end{proposition}

\begin{proof}
All expectations and probabilities are over
$\boldsymbol D_m$.
If $\mathcal H_m$ is empty, then
$M_m(\boldsymbol D_m)=M_{m,0}=1$, and both bounds hold.
We therefore consider the case
$Z_m^{\mathcal H}(\mathbf 0)>0$.

Define
\begin{align*}
    \psi_m(z)
    \triangleq
    \frac{Z_m^{\mathcal U}}{Z_m^{\mathcal U}+z},
    \qquad z\geq0.
\end{align*}
Since $\mathcal U_m$ is nonempty,
$Z_m^{\mathcal U}>0$, and
\begin{align*}
    \psi_m'(z)
    &=
    -\frac{Z_m^{\mathcal U}}
    {(Z_m^{\mathcal U}+z)^2}<0,\\
    \psi_m''(z)
    &=
    \frac{2Z_m^{\mathcal U}}
    {(Z_m^{\mathcal U}+z)^3}>0.
\end{align*}
Thus $\psi_m$ is decreasing and convex.

\emph{Expected attention mass.}
By Jensen's inequality and
Eq.~\eqref{eq:app:partition-contraction},
\begin{align*}
    \mathbb E M_m(\boldsymbol D_m)
    &=
    \mathbb E\psi_m
    \left(Z_m^{\mathcal H}(\boldsymbol D_m)\right)\\
    &\geq
    \psi_m\left(
        \mathbb E Z_m^{\mathcal H}(\boldsymbol D_m)
    \right)\\
    &\geq
    \psi_m\left(
        e^{-\kappa}Z_m^{\mathcal H}(\mathbf 0)
    \right)\\
    &=
    \frac{Z_m^{\mathcal U}}
    {Z_m^{\mathcal U}
     +e^{-\kappa}Z_m^{\mathcal H}(\mathbf 0)}.
\end{align*}
The native mass satisfies
\begin{align*}
    \frac{Z_m^{\mathcal H}(\mathbf 0)}
    {Z_m^{\mathcal U}}
    =
    \frac{1-M_{m,0}}{M_{m,0}}.
\end{align*}
Substituting this identity proves
Eq.~\eqref{eq:app:attention-mass}.
Since $\kappa>0$ and
$Z_m^{\mathcal H}(\mathbf 0)>0$, the lower bound strictly
exceeds
\begin{align*}
    \frac{Z_m^{\mathcal U}}
    {Z_m^{\mathcal U}+Z_m^{\mathcal H}(\mathbf 0)}
    =
    M_{m,0}.
\end{align*}

\emph{High-probability attention mass.}
Applying Markov's inequality to the nonnegative random
variable $Z_m^{\mathcal H}(\boldsymbol D_m)$ gives
\begin{align*}
    &\Pr\left(
        Z_m^{\mathcal H}(\boldsymbol D_m)
        >
        \frac{e^{-\kappa}
        Z_m^{\mathcal H}(\mathbf 0)}{\eta}
    \right)\\
    &\quad\leq
    \frac{
        \eta\,\mathbb E Z_m^{\mathcal H}(\boldsymbol D_m)
    }{
        e^{-\kappa}Z_m^{\mathcal H}(\mathbf 0)
    }
    \leq\eta.
\end{align*}
On the complementary event, which has probability at
least $1-\eta$, monotonicity of $\psi_m$ yields
\begin{align*}
    M_m(\boldsymbol D_m)
    &\geq
    \psi_m\left(
        \frac{e^{-\kappa}
        Z_m^{\mathcal H}(\mathbf 0)}{\eta}
    \right)\\
    &=
    \frac{1}
    {1+e^{-\kappa}(1-M_{m,0})/(\eta M_{m,0})}.
\end{align*}
This proves Eq.~\eqref{eq:app:attention-probability}.
For $0<M_{m,0}<1$, this bound exceeds $M_{m,0}$ precisely
when $e^{-\kappa}/\eta<1$, or equivalently
$\kappa>\log(1/\eta)$.
\end{proof}

\paragraph{A sufficient condition from spectral contraction.}

The condition in Eq.~\eqref{eq:app:partition-contraction}
can be expressed in terms of the mean decrease and
fluctuations of each perturbed logit.
For $s\in\mathcal H_m$, define
\begin{align*}
    \Gamma_{ms}
    \triangleq
    a_{ts}
    -
    \mathbb E_{\boldsymbol D_m}
    \widetilde a_{ts}(\boldsymbol D_m).
\end{align*}
For a key $s$ in block $b$, Proposition~\ref{prop:app:contraction}
gives
\begin{align*}
    \Gamma_{ms}
    =
    \frac{1}{\sqrt{d_h}}
    \sum_{r=1}^{R}
    (1-\rho_{mbr})
    \bar{\qv}_{t,r}^{\top}
    \operatorname{Rot}(\omega_r(s-t))
    \bar{\kv}_{s,r}.
\end{align*}

Suppose that each centered logit satisfies the
exponential-moment bound
\begin{align*}
    \mathbb E_{\boldsymbol D_m}
    \exp\left(
        \widetilde a_{ts}(\boldsymbol D_m)
        -
        \mathbb E_{\boldsymbol D_m}
        \widetilde a_{ts}(\boldsymbol D_m)
    \right)
    \leq
    \exp(v_{ms}^2/2)
\end{align*}
for some $v_{ms}\geq0$.
Then
\begin{align*}
    \mathbb E_{\boldsymbol D_m}
    e^{\widetilde a_{ts}(\boldsymbol D_m)}
    &=
    \exp\left(
        \mathbb E_{\boldsymbol D_m}
        \widetilde a_{ts}(\boldsymbol D_m)
    \right)\\
    &\quad{}\times
    \mathbb E_{\boldsymbol D_m}
    \exp\left(
        \widetilde a_{ts}(\boldsymbol D_m)
        -
        \mathbb E_{\boldsymbol D_m}
        \widetilde a_{ts}(\boldsymbol D_m)
    \right)\\
    &\leq
    \exp\left(a_{ts}-\Gamma_{ms}+v_{ms}^2/2\right).
\end{align*}
Consequently, a sufficient condition for
Eq.~\eqref{eq:app:partition-contraction} is
\begin{align*}
    \Gamma_{ms}-\frac{v_{ms}^2}{2}\geq\kappa
    \qquad
    \text{for every }s\in\mathcal H_m.
\end{align*}
Indeed, summing the exponential-moment bounds over
perturbed keys gives
\begin{align*}
    \mathbb E_{\boldsymbol D_m}
    Z_m^{\mathcal H}(\boldsymbol D_m)
    &=
    \sum_{s\in\mathcal H_m}
    \mathbb E_{\boldsymbol D_m}
    e^{\widetilde a_{ts}(\boldsymbol D_m)}\\
    &\leq
    \sum_{s\in\mathcal H_m}
    e^{a_{ts}}
    \exp\left(-\Gamma_{ms}+v_{ms}^2/2\right)\\
    &\leq
    e^{-\kappa}
    \sum_{s\in\mathcal H_m}e^{a_{ts}}\\
    &=
    e^{-\kappa}Z_m^{\mathcal H}(\mathbf 0).
\end{align*}

\section{Method Details}

The \name pipeline has three components omitted from the main body:
the context perturbation settings,
the anchor and carrier construction for the LTE first stage,
and the shallow weight write that precedes the LTE block fitting.

\subsection{Context Perturbation Settings}
\label{app:offsets}

For a target sentence $m$, the prompt and each preceding sentence form
separate blocks, which use separate perturbations at the same scale $\sigma$.
Each block shares one rounded-Gaussian perturbation across attention layers.
Draws are independent across blocks and refreshed at each optimization step. The target sentence and its first-token readout position are excluded from the perturbed key positions. 

\subsection{LTE Stage 1: Anchor and Target Construction}
\label{app:lte-targets}

We now describe the anchor states used to optimize the residuals
in Eq.~\eqref{eq:krp:delta}. All selection scores and source states are
computed with the pre-edited model under native RoPE.

Our design is motivated by the intuition that mentioning a concept can evoke
several associated pieces of information in parallel. We approximate this
effect by treating each sentence as an atomic editing unit and binding all
sentence-level targets to shared query-side anchors through a batched write.
We find that the construction below works well empirically.

\textbf{Query anchors.}
Let $\xv=(x_1,\ldots,x_n)$ denote the tokenized editing prompt. We compare
its token probabilities under the original model $\pi_{\theta_0}$ and under
$\pi_{\theta_0}^{\cv}$, which additionally receives the editing passage in
context. Within their common prompt suffix, the log-probability lift is
\begin{align}
    g_p
    =
    \log\pi_{\theta_0}^{\cv}(x_p\mid\xv_{<p})
    -
    \log\pi_{\theta_0}(x_p\mid\xv_{<p}).
    \label{eq:krp:lift}
\end{align}
We select positions with $g_p>\tau_g$ and
$\pi_{\theta_0}^{\cv}(x_p\mid\xv_{<p})>\kappa_g$ to form $\mathcal A_q$,
excluding the final prompt position and requiring the resulting set to be
nonempty. The lift identifies prompt tokens supported by the passage, and
the confidence threshold removes uncertain predictions. We use these
positions as a query-side proxy for the entity or concept being edited.

\textbf{Sentence-level targets.}
For each sentence $m$, we augment the shared query anchors with a temporary
carrier position $u$ and define
$\mathcal A_m=\mathcal A_q\cup\{u\}$. To make the carrier
sentence-specific, we select the content token whose prediction benefits most
from the query and record its native input $\sv_m$ to the layer-$z$ MLP
\texttt{down\_proj}. During the optimization of $\delta_m$, we place $\sv_m$
at the carrier and add $\delta_m$ to the layer-$z$ MLP outputs at
$\mathcal A_m$, following Eq.~\eqref{eq:krp:delta}. The complete passage is
teacher-forced, with loss only on $\mathcal I_m$ and key perturbations applied
to the preceding context. Sentences without an eligible content token are
deferred to Stage 2, and the temporary carrier is absent at evaluation.

After optimizing the sentence-level residuals, we disable key perturbation,
record the corresponding anchor outputs plus $\delta_m^\star$, and jointly
write all resulting targets into the shallow MLP layers as described in
Appendix~\ref{app:stage1-write}.

\subsection{LTE Weight Writing}
\label{app:stage1-write}

\textbf{Shallow-layer write.}
Let $\mathcal L_K=(\ell_0,\ldots,\ell_{R_K-1})$ be the shallow write
layers, and concatenate the sentence targets into $\mathbf Z^\star$.
At layer $\ell=\ell_r$, let $\mathbf K_e^\ell$ collect the MLP \texttt{down\_proj} input keys
at the anchors, $\mathbf K_p^\ell$ the preservation keys at the remaining
prompt positions, and $\mathbf Z^\ell$ the current hidden outputs at the final target layer
$z$ for the same anchor columns. We distribute the remaining residual over
the current and subsequent write layers:
\begin{align*}
    \mathbf E^\ell
    =\frac{\mathbf Z^\star-\mathbf Z^\ell}{R_K-r}.
\end{align*}
Using AlphaEdit's prior-knowledge preservation projector
$\mathbf P_\ell$~\citep{fang2025alphaedit}, we form
\begin{align}
    \mathbf A_\ell
      &=\mathbf P_\ell\bigl(\mathbf K_e^\ell\mathbf K_e^{\ell\top}
          +\mathbf K_p^\ell\mathbf K_p^{\ell\top}\bigr)
          +\lambda_2\mathbf I,\\
    \mathbf B_\ell&=\mathbf P_\ell\mathbf K_e^\ell\mathbf E^{\ell\top},
    \qquad
    \mathbf A_\ell(\Delta\mathbf W^\ell)^\top=\mathbf B_\ell.
    \label{eq:krp:write}
\end{align}
Here $\mathbf W_0^\ell$ is the pre-edit weight and $\lambda_2$ is the
ridge coefficient. The update is
$\mathbf W^\ell\leftarrow\mathbf W_0^\ell+\Delta\mathbf W^\ell$.
We recompute the keys and current target-layer
outputs after each layer's write, using the settings in
Appendix~\ref{app:hyperparams}.

\textbf{Passage-level write.}
After the shallow write, we record the sentence-anchor outputs at the
second-stage fitting layer. For $\mu$KE and UnKE, these outputs supply the
reference states for an additional MSE preservation loss during block
fitting, alongside the editor's passage-target and locality losses.
The second stage uses the complete passage, including sentences for which
no first-stage carrier state was selected.

\section{Implementation Details}
\label{app:setup}

\subsection{Training Details and Hyperparameters}
\label{app:hyperparams}

\textbf{Backbones.}
Tab.~\ref{tab:hp:backbones} lists the checkpoints. Models are loaded in
\texttt{bfloat16}, and prompts use each tokenizer's native chat template
during both editing and evaluation. 

\begin{table}[htbp]
\centering
\begingroup
\footnotesize
\setlength{\tabcolsep}{5pt}
\caption{Backbone checkpoints.}
\label{tab:hp:backbones}
\begin{tabular}{>{\bfseries}r l c}
\toprule[0.4ex]
\textbf{Name} & \textbf{Checkpoint} & \textbf{Precision} \\
\midrule[0.2ex]
Qwen2.5 & \texttt{Qwen/Qwen2.5-7B-Instruct} & bf16 \\
Llama3.1 & \texttt{meta-llama/Llama-3.1-8B-Instruct} & bf16 \\
\bottomrule[0.4ex]
\end{tabular}
\endgroup
\end{table}

\textbf{\name with LTE.}
The two-stage configuration first optimizes a per-sentence hidden residual
under key perturbation, then writes the results into shallow MLP layers
before handing off to the LTE editor.
The shallow write uses AlphaEdit~\citep{fang2025alphaedit} with a
null-space projector computed from Wikipedia samples.
Tab.~\ref{tab:hp:krp-lte} lists all settings.
For both stages, target losses are read from the final transformer layer
($27$ for Qwen2.5, $31$ for Llama3.1).
Perturbations are disabled when recording native target states and during
evaluation.

\begin{table}[htbp]
\centering
\begingroup
\footnotesize
\setlength{\tabcolsep}{5pt}
\renewcommand{\arraystretch}{1.05}
\caption{\name with LTE configuration.}
\label{tab:hp:krp-lte}
\begin{tabular}{>{\bfseries}r c}
\toprule[0.4ex]
\textbf{Hyperparameter} & \textbf{Value} \\
\midrule[0.2ex]
\multicolumn{2}{l}{\textbf{\textit{Perturbation \& anchors}}} \\
Perturbation scale $\sigma$ & $16$ \\
Query-anchor lift threshold $\tau_g$ & $1.0$ \\
Query-anchor confidence threshold $\kappa_g$ & $0.3$ \\
Carrier \& target optimization layer $z$ & $6$ \\
\midrule[0.2ex]
\multicolumn{2}{l}{\textbf{\textit{First stage: sentence residual}}} \\
Optimizer & Adam \\
Learning rate & $0.5$ \\
Loss evaluations / gradient updates & $25$ / $24$ \\
Residual norm cap (relative to initial state) & $4.0$ \\
\midrule[0.2ex]
\multicolumn{2}{l}{\textbf{\textit{Shallow write (AlphaEdit)}}} \\
Write layers & $4$--$6$ \\
Ridge coefficient $\lambda_2$ & $10$ \\
Prompt preservation keys & $\le 400$ \\
Null-space projector: Wikipedia samples & $20{,}000$ \\
Null-space projector: singular-value threshold & $0.02$ \\
\bottomrule[0.4ex]
\end{tabular}
\endgroup
\end{table}

\textbf{FT-M.}
FT-M optimizes only the \texttt{down\_proj} weight in layer $8$.
Plain FT-M trains on the full-passage loss without perturbation.
The \name variant enables sentence chunking with $\sigma=1$:
each target sentence is evaluated with its own offsets, and the
sentence losses are averaged in proportion to token counts.
Tab.~\ref{tab:hp:ft} lists the remaining settings.

\begin{table}[htbp]
\centering
\begingroup
\footnotesize
\setlength{\tabcolsep}{5pt}
\renewcommand{\arraystretch}{1.05}
\caption{FT-M configuration. Perturbation and sentence-loss
settings apply to the \name variant; plain FT-M disables perturbations
and uses the full-passage loss.}
\label{tab:hp:ft}
\begin{tabular}{>{\bfseries}r c}
\toprule[0.4ex]
\textbf{Hyperparameter} & \textbf{Value} \\
\midrule[0.2ex]
Trainable parameters & Layer-$8$ MLP \texttt{down\_proj} weight \\
Optimizer & Adam \\
Learning rate & $1\times10^{-4}$ \\
Gradient updates per request & $40$ \\
Weight decay & $0$ \\
Perturbation scale $\sigma$ (\name) & $1.0$ \\
Sentence-loss reduction (\name) & Token-weighted average \\
\bottomrule[0.4ex]
\end{tabular}
\endgroup
\end{table}

\textbf{LoRA.}
LoRA trains rank-$r$ adapters on every linear projection of
every transformer layer, keeping the base weights frozen.
Adapters remain active during evaluation and are removed afterward.
The training objective follows FT-M: plain LoRA uses the
full-passage loss without perturbation, and LoRA with \name uses
sentence chunking with $\sigma=1$.
Tab.~\ref{tab:hp:lora} lists the per-backbone settings.

\begin{table}[htbp]
\centering
\begingroup
\footnotesize
\setlength{\tabcolsep}{5pt}
\renewcommand{\arraystretch}{1.05}
\caption{LoRA configuration. Objective, chunking, and
perturbation settings follow Tab.~\ref{tab:hp:ft}.}
\label{tab:hp:lora}
\begin{tabular}{>{\bfseries}r c c}
\toprule[0.4ex]
\textbf{Hyperparameter} & \textbf{Qwen2.5} & \textbf{Llama3.1} \\
\midrule[0.2ex]
Adapter rank $r$ & $8$ & $12$ \\
Scaling $\alpha$ & $32$ & $24$ \\
Learning rate (AdamW) & $1\times10^{-4}$ & $5\times10^{-5}$ \\
Gradient updates per request & $40$ & $40$ \\
Weight decay & $0$ & $0$ \\
\bottomrule[0.4ex]
\end{tabular}
\endgroup
\end{table}

\textbf{COIN$^\star$.}
COIN$^{\star}$ uses the context-alignment objective of
\citet{zhou2026context} with the same write backend as FT-M, optimizing only
the \texttt{down\_proj} weight of layer $8$.
The objective is
\begin{equation*}
\mathcal{L}_{\text{COIN}^{\star}}
= \mathcal{L}_{\text{NLL}}
+ \alpha\,\frac{1}{|S|}\sum_{t\in S}
\mathrm{KL}\!\left(
p_\theta(\cdot \mid q, y_{<t})\,\big\|\,
p_\theta(\cdot \mid y_{t-k:t-1})
\right),
\end{equation*}
where the prediction from the last $k$ passage tokens is trained to match the
full-context prediction, which is held fixed as the teacher.
We omit COIN's original locality-preserving mechanisms, including the key-consistency regularizer and early-stopping rule. These mechanisms can also be applied to the other editors and would complicate the comparison.

When combining COIN$^{\star}$ with \name, we perturb only the NLL term and
leave the alignment term unchanged. Sentence chunking and offset settings
follow FT-M.
Tab.~\ref{tab:hp:coin} lists the per-backbone settings. For Llama3.1, we use FT-M's learning rate and number of steps, selected based on the performance of base COIN$^{\star}$ on 100 randomly sampled UnFine-UnKE examples.

\begin{table}[htbp]
\centering
\begingroup
\footnotesize
\setlength{\tabcolsep}{5pt}
\renewcommand{\arraystretch}{1.05}
\caption{COIN$^\star$ configuration. Perturbation and chunking
settings for the \name variant follow Tab.~\ref{tab:hp:ft}.}
\label{tab:hp:coin}
\begin{tabular}{>{\bfseries}r c c}
\toprule[0.4ex]
\textbf{Hyperparameter} & \textbf{Qwen2.5} & \textbf{Llama3.1} \\
\midrule[0.2ex]
Trainable parameters & \multicolumn{2}{c}{Layer-$8$ MLP \texttt{down\_proj} weight} \\
Local window $k$ & $5$ & $5$ \\
Alignment coefficient $\alpha$ & $0.1$ & $0.1$ \\
Aligned positions per request & $\le 12$ & $\le 12$ \\
Optimizer & AdamW & AdamW \\
Learning rate & $5\times10^{-4}$ & $1\times10^{-4}$ \\
Gradient updates per request & $25$ & $40$ \\
Weight decay & $0$ & $0$ \\
\bottomrule[0.4ex]
\end{tabular}
\endgroup
\end{table}

\textbf{LTE editors.}
Tab.~\ref{tab:hp:lte} gives the configurations of $\mu$KE and UnKE.
Both optimize hidden targets with Adam and fit the entire transformer
block at layer $7$ with AdamW.
$\mu$KE uses non-overlapping $20$-token windows; UnKE constructs one
target at the final prompt position for the complete passage.

\begin{table}[htbp]
\centering
\begingroup
\footnotesize
\setlength{\tabcolsep}{5pt}
\renewcommand{\arraystretch}{1.05}
\caption{LTE editor configurations, retained when they
serve as the second-stage writer in \name.}
\label{tab:hp:lte}
\begin{tabular}{>{\bfseries}r c c}
\toprule[0.4ex]
\textbf{Hyperparameter} & $\boldsymbol{\mu}\textbf{KE}$ & \textbf{UnKE} \\
\midrule[0.2ex]
Trainable transformer block & $7$ & $7$ \\
Hidden-target learning rate (Adam) & $0.5$ & $0.5$ \\
Maximum target-loss evaluations & $25$ & $25$ \\
Hidden-residual penalty coefficient & $10^{-3}$ & $10^{-3}$ \\
Residual norm cap (relative to initial state) & $4.0$ & $4.0$ \\
Target-loss stopping threshold & $5\times10^{-3}$ & --- \\
Writer learning rate (AdamW) & $2\times10^{-4}$ & $2\times10^{-4}$ \\
Writer gradient updates & $50$ & $50$ \\
Writer weight decay & $0.01$ & $0.01$ \\
Locality reference sentences & $3$ & $3$ \\
Target window size / overlap (tokens) & $20$ / $0$ & Full passage \\
\bottomrule[0.4ex]
\end{tabular}
\endgroup
\end{table}

The block-fitting loss sums an editing MSE and a locality MSE on three
fixed, unrelated reference sentences supplied by the evaluation harness.
With \name, it also includes an MSE preserving the first-stage anchor
outputs at layer $7$; all three terms have unit coefficients.

\subsection{Metric Computation}
\label{app:metrics}

\textbf{Generation.}
All evaluations query the edited model without supplying the editing passage.
We use greedy decoding and cap D3 generations at $30$ new tokens, with the
instruction ``Answer directly and keep your answer within 20 words.''
The original and paraphrased queries (Orig Q/Reph Q) carry no additional instruction.

\textbf{Holistic recall.}
We compute ROUGE-L recall against the complete editing passage.
For each request, we report the original query (Orig Q), its paraphrase (Reph Q),
and their average as an overall measure.

\textbf{Atomic-fact recall.}
For a request with $J$ atomic questions, let $a_j$ be the reference answer,
$\hat a_j$ the generated answer, and $\mathcal P_j$ the annotated key phrases
for question $j$. We average ROUGE-L recall over the $J$ questions.
Following the UnFine implementation~\citep{wang2026fable}, HR counts a key
phrase as recovered only when all its words appear in order in the generation, and LCS computes the word-level longest common subsequence
length:
\begin{align*}
    \mathrm{HR}
    &= \frac{\sum_{j=1}^{J}\sum_{p\in\mathcal P_j}
        \mathbb{I}\!\left[\operatorname{LCS}(p,\hat a_j)=|p|\right]}
        {\sum_{j=1}^{J}|\mathcal P_j|}, \\
    C_{\mathrm{LCS}}
    &= \frac{1}{J}\sum_{j=1}^{J}
        \frac{\operatorname{LCS}(a_j,\hat a_j)}{|a_j|}.
\end{align*}

We follow \citet{wang2026fable} for the computation of the metrics. Namely, ROUGE-L is computed with the \texttt{rouge} package's default preprocessing.
For HR and $C_{\mathrm{LCS}}$, we lowercase the text and extract alphanumeric words
using \texttt{[A-Za-z0-9]+}; generated answers are also stripped of chat end markers
and trailing periods.

\subsection{General Capability Evaluation Datasets}
\label{app:general-capability}
Following~\citet{wang2026fable, zhou2026context},
we evaluate post-editing general capabilities on six benchmarks:
SST-2~\citep{socher2013recursive} (sentiment),
MMLU~\citep{hendrycks2021mmlu} (knowledge and reasoning),
MRPC~\citep{dolan2005mrpc} (paraphrase detection),
CoLA~\citep{warstadt2019cola} (grammatical acceptability),
RTE~\citep{bentivogli2009rte} (textual entailment),
and NLI~\citep{williams2018multinli} (natural language inference).
We report likelihood-based F1 scores using the codebase of \citet{wang2026fable}.

\section{Additional Results}

This section presents results omitted from the main text.

\subsection{Sentence-Order Effect for Pre-Edit Loss}
\label{app:shuffle-control}

In this section, we study the effect of sentence order on
pre-edit loss to test whether the positional decline is caused by intrinsic
sentence difficulty. We randomly shuffle the sentences within each UnFine
passage and recompute their pre-edit losses using Qwen2.5-7B-Instruct. Since
shuffling can make the original editing query inconsistent with the reordered
passage, we use the order-insensitive prompt \texttt{Introduce X} for both the
original and shuffled conditions, where \texttt{X} is the edited subject.
Results are shown in Fig.~\ref{fig:shuffle-loss}.

Across all three datasets, pre-edit loss still decreases with
passage position after shuffling. This rules out sentence
order as the cause and further supports \emph{context-induced difficulty underestimation}.

\begin{figure}[H]
    \centering
    \includegraphics[width=\textwidth]{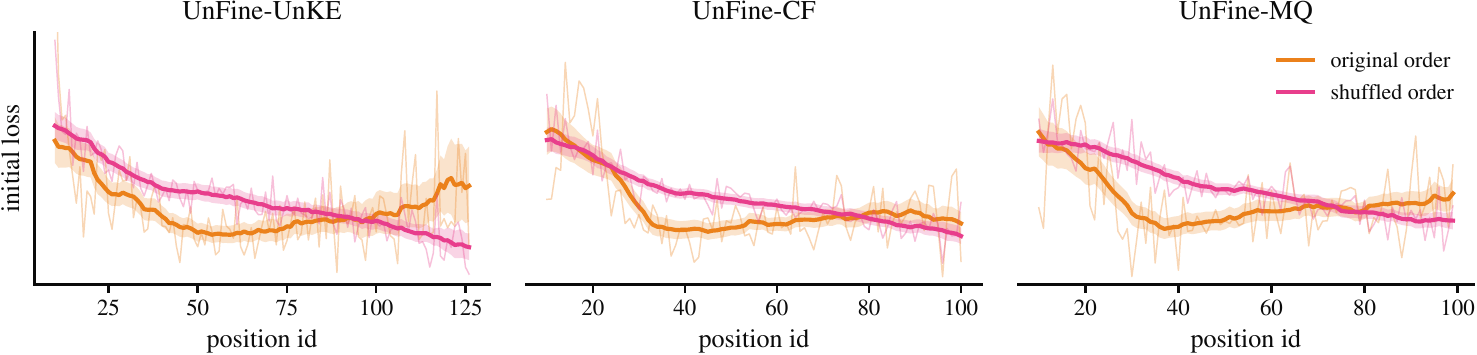}
    \caption{Pre-edit loss by passage position under the original and randomly
    shuffled sentence orders. Both conditions use the order-insensitive prompt
    \texttt{Introduce X}. The positional decline persists after shuffling.}
    \label{fig:shuffle-loss}
\end{figure}

\subsection{Atomic-Fact Recall by Question Position}
\label{app:atomic-positions}

Tab.~\ref{tab:atomic-positions} reports the per-question-position
$C_{\mathrm{LCS}}$ scores underlying Figure~\ref{fig:qidx-depth}, and
Tab.~\ref{tab:atomic-positions-hr} reports Hit Rate for the same settings.
Q5--Q7 are listed separately for UnFine-UnKE, whereas UnFine-CF and
UnFine-MQ contain only Q1--Q5. Figure~\ref{fig:qidx-depth} pools Q5--Q7
using sample-count weights within each editor and variant.
Scores are reported to one decimal place.

\begin{table}[!htbp]
\centering
\caption{
Atomic-fact recall ($C_{\mathrm{LCS}}$) by question position for both backbones.
Missing entries are ``--''.
}
\label{tab:atomic-positions}
\begingroup
\resizebox{0.7\linewidth}{!}{%
\begin{tabular}{>{\bfseries}r ccccccc ccccccc}
\toprule[0.4ex]
& \multicolumn{7}{c}{\bfseries Qwen2.5} & \multicolumn{7}{c}{\bfseries Llama3.1} \\
\cmidrule(lr){2-8}\cmidrule(lr){9-15}
\textbf{Method} & \textbf{Q1} & \textbf{Q2} & \textbf{Q3} & \textbf{Q4} & \textbf{Q5} & \textbf{Q6} & \textbf{Q7} & \textbf{Q1} & \textbf{Q2} & \textbf{Q3} & \textbf{Q4} & \textbf{Q5} & \textbf{Q6} & \textbf{Q7} \\
\midrule[0.2ex]
\multicolumn{15}{c}{\bfseries UnFine-UnKE} \\
\midrule[0.2ex]
LoRA & 59.3 & 61.4 & 51.7 & 45.5 & 38.4 & 32.7 & 24.2 & 70.6 & 63.9 & 56.3 & 47.6 & 42.9 & 40.2 & 38.4 \\
+Ours & 85.0 & 78.4 & 67.3 & 61.3 & 54.1 & 44.9 & 36.3 & 84.9 & 73.3 & 62.0 & 57.2 & 51.5 & 51.9 & 44.7 \\
\noalign{\vskip 0.5ex}\cdashline{2-15}\noalign{\vskip 0.5ex}
COIN$^{\star}$ & 65.9 & 55.1 & 44.3 & 40.7 & 36.5 & 35.0 & 28.0 & 66.6 & 58.4 & 49.6 & 46.2 & 42.1 & 36.6 & 35.9 \\
+Ours & 75.9 & 64.4 & 54.3 & 52.8 & 46.2 & 41.4 & 35.7 & 75.1 & 68.0 & 57.3 & 51.5 & 46.0 & 40.6 & 34.8 \\
\noalign{\vskip 0.5ex}\cdashline{2-15}\noalign{\vskip 0.5ex}
FT-M & 33.8 & 31.7 & 29.8 & 28.1 & 25.6 & 17.9 & 20.8 & 65.1 & 57.6 & 48.1 & 43.4 & 40.5 & 30.4 & 34.1 \\
+Ours & 40.3 & 36.7 & 35.7 & 32.2 & 29.3 & 23.3 & 26.6 & 74.5 & 67.2 & 56.8 & 51.6 & 45.6 & 44.3 & 46.5 \\
\noalign{\vskip 0.5ex}\cdashline{2-15}\noalign{\vskip 0.5ex}
$\mu$KE & 37.2 & 36.6 & 34.2 & 30.9 & 29.1 & 25.1 & 35.6 & 43.6 & 41.1 & 37.7 & 33.0 & 31.0 & 22.6 & 24.1 \\
+Ours & 52.6 & 42.6 & 41.1 & 41.7 & 40.4 & 29.3 & 47.3 & 54.0 & 43.2 & 40.2 & 38.9 & 37.1 & 32.1 & 29.4 \\
\noalign{\vskip 0.5ex}\cdashline{2-15}\noalign{\vskip 0.5ex}
UnKE & 36.4 & 37.6 & 33.6 & 31.1 & 29.8 & 23.8 & 23.2 & 40.8 & 35.2 & 33.4 & 31.5 & 28.3 & 25.8 & 35.4 \\
+Ours & 53.3 & 42.9 & 39.1 & 39.4 & 40.8 & 25.2 & 32.5 & 58.5 & 45.2 & 39.3 & 37.5 & 37.6 & 35.1 & 31.6 \\
\midrule[0.2ex]
\multicolumn{15}{c}{\bfseries UnFine-CF} \\
\midrule[0.2ex]
LoRA & 79.3 & 75.9 & 70.6 & 61.5 & 54.5 & -- & -- & 86.7 & 85.3 & 79.2 & 72.5 & 65.6 & -- & -- \\
+Ours & 89.4 & 87.6 & 81.5 & 75.9 & 69.8 & -- & -- & 91.4 & 88.4 & 81.2 & 78.3 & 70.3 & -- & -- \\
\noalign{\vskip 0.5ex}\cdashline{2-15}\noalign{\vskip 0.5ex}
COIN$^{\star}$ & 81.3 & 72.3 & 63.1 & 59.5 & 50.0 & -- & -- & 81.2 & 79.5 & 70.8 & 65.6 & 58.9 & -- & -- \\
+Ours & 80.9 & 76.8 & 69.6 & 60.2 & 54.2 & -- & -- & 84.3 & 81.6 & 76.0 & 68.6 & 62.3 & -- & -- \\
\noalign{\vskip 0.5ex}\cdashline{2-15}\noalign{\vskip 0.5ex}
FT-M & 59.9 & 45.8 & 45.3 & 42.8 & 35.3 & -- & -- & 82.0 & 79.1 & 69.8 & 63.9 & 57.1 & -- & -- \\
+Ours & 63.3 & 54.6 & 53.3 & 44.5 & 43.0 & -- & -- & 85.1 & 81.6 & 74.7 & 68.2 & 62.3 & -- & -- \\
\noalign{\vskip 0.5ex}\cdashline{2-15}\noalign{\vskip 0.5ex}
$\mu$KE & 50.5 & 37.6 & 41.7 & 40.8 & 39.1 & -- & -- & 62.6 & 56.2 & 54.0 & 47.1 & 46.2 & -- & -- \\
+Ours & 62.0 & 49.8 & 46.8 & 49.0 & 43.5 & -- & -- & 61.4 & 52.5 & 51.0 & 47.9 & 46.1 & -- & -- \\
\noalign{\vskip 0.5ex}\cdashline{2-15}\noalign{\vskip 0.5ex}
UnKE & 49.0 & 38.0 & 38.4 & 42.4 & 38.0 & -- & -- & 54.6 & 50.9 & 48.7 & 44.1 & 43.0 & -- & -- \\
+Ours & 63.1 & 48.5 & 48.1 & 44.9 & 45.6 & -- & -- & 62.0 & 52.6 & 50.5 & 48.2 & 43.8 & -- & -- \\
\midrule[0.2ex]
\multicolumn{15}{c}{\bfseries UnFine-MQ} \\
\midrule[0.2ex]
LoRA & 71.2 & 70.9 & 66.2 & 61.8 & 59.6 & -- & -- & 78.9 & 78.9 & 71.4 & 69.0 & 65.7 & -- & -- \\
+Ours & 85.9 & 86.7 & 78.9 & 69.9 & 69.0 & -- & -- & 85.6 & 85.4 & 75.0 & 73.2 & 72.5 & -- & -- \\
\noalign{\vskip 0.5ex}\cdashline{2-15}\noalign{\vskip 0.5ex}
COIN$^{\star}$ & 73.3 & 67.2 & 58.6 & 58.6 & 56.3 & -- & -- & 72.9 & 74.3 & 62.0 & 59.1 & 57.2 & -- & -- \\
+Ours & 81.4 & 70.6 & 64.5 & 63.3 & 59.0 & -- & -- & 80.5 & 76.3 & 67.8 & 63.4 & 61.8 & -- & -- \\
\noalign{\vskip 0.5ex}\cdashline{2-15}\noalign{\vskip 0.5ex}
FT-M & 55.2 & 51.4 & 50.4 & 41.8 & 41.8 & -- & -- & 72.3 & 73.8 & 62.4 & 60.5 & 58.0 & -- & -- \\
+Ours & 58.0 & 58.2 & 53.6 & 48.5 & 44.4 & -- & -- & 78.5 & 74.9 & 68.3 & 61.3 & 59.9 & -- & -- \\
\noalign{\vskip 0.5ex}\cdashline{2-15}\noalign{\vskip 0.5ex}
$\mu$KE & 45.2 & 48.9 & 44.5 & 43.9 & 42.7 & -- & -- & 57.0 & 56.0 & 45.5 & 46.9 & 44.7 & -- & -- \\
+Ours & 56.4 & 49.9 & 46.5 & 45.3 & 44.1 & -- & -- & 57.5 & 57.8 & 51.7 & 46.0 & 48.3 & -- & -- \\
\noalign{\vskip 0.5ex}\cdashline{2-15}\noalign{\vskip 0.5ex}
UnKE & 45.8 & 46.8 & 41.1 & 42.4 & 43.0 & -- & -- & 53.3 & 52.5 & 48.3 & 46.9 & 45.5 & -- & -- \\
+Ours & 55.4 & 45.3 & 41.5 & 46.2 & 42.8 & -- & -- & 56.5 & 55.6 & 48.9 & 49.8 & 48.0 & -- & -- \\
\bottomrule[0.4ex]
\end{tabular}%
}
\endgroup
\end{table}

\begin{table}[!htbp]
\centering
\caption{
Atomic-fact recall (Hit Rate) by question position for both backbones.
Missing entries are ``--''.
}
\label{tab:atomic-positions-hr}
\begingroup
\fontsize{9}{10.5}\selectfont
\resizebox{0.7\linewidth}{!}{%
\begin{tabular}{>{\bfseries}r ccccccc ccccccc}
\toprule[0.4ex]
& \multicolumn{7}{c}{\bfseries Qwen2.5} & \multicolumn{7}{c}{\bfseries Llama3.1} \\
\cmidrule(lr){2-8}\cmidrule(lr){9-15}
\textbf{Method} & \textbf{Q1} & \textbf{Q2} & \textbf{Q3} & \textbf{Q4} & \textbf{Q5} & \textbf{Q6} & \textbf{Q7} & \textbf{Q1} & \textbf{Q2} & \textbf{Q3} & \textbf{Q4} & \textbf{Q5} & \textbf{Q6} & \textbf{Q7} \\
\midrule[0.2ex]
\multicolumn{15}{c}{\bfseries UnFine-UnKE} \\
\midrule[0.2ex]
LoRA & 49.8 & 49.7 & 40.1 & 31.6 & 25.0 & 16.7 & 15.4 & 59.1 & 51.7 & 43.1 & 33.0 & 26.3 & 21.2 & 23.1 \\
+Ours & 80.2 & 71.3 & 58.9 & 50.8 & 43.2 & 27.3 & 26.9 & 78.3 & 65.1 & 51.8 & 44.5 & 39.4 & 36.4 & 19.2 \\
\noalign{\vskip 0.5ex}\cdashline{2-15}\noalign{\vskip 0.5ex}
COIN$^{\star}$ & 56.1 & 42.4 & 29.1 & 23.2 & 18.0 & 15.2 & 15.4 & 53.8 & 41.2 & 31.4 & 28.4 & 24.0 & 16.7 & 30.8 \\
+Ours & 69.1 & 55.9 & 41.2 & 38.3 & 30.5 & 24.2 & 19.2 & 64.1 & 53.4 & 42.5 & 36.0 & 27.0 & 22.7 & 26.9 \\
\noalign{\vskip 0.5ex}\cdashline{2-15}\noalign{\vskip 0.5ex}
FT-M & 22.3 & 14.9 & 14.9 & 13.5 & 13.6 & 4.5 & 15.4 & 52.0 & 39.8 & 29.4 & 25.6 & 22.4 & 9.1 & 30.8 \\
+Ours & 29.2 & 20.4 & 19.6 & 17.4 & 17.9 & 6.1 & 11.5 & 63.7 & 52.7 & 41.2 & 35.4 & 28.0 & 25.8 & 34.6 \\
\noalign{\vskip 0.5ex}\cdashline{2-15}\noalign{\vskip 0.5ex}
$\mu$KE & 18.5 & 16.8 & 15.3 & 11.3 & 10.0 & 6.1 & 19.2 & 27.7 & 21.3 & 20.2 & 13.9 & 12.5 & 6.1 & 7.7 \\
+Ours & 32.0 & 23.6 & 21.4 & 23.2 & 21.1 & 9.1 & 23.1 & 34.7 & 24.8 & 19.2 & 19.9 & 18.9 & 12.1 & 19.2 \\
\noalign{\vskip 0.5ex}\cdashline{2-15}\noalign{\vskip 0.5ex}
UnKE & 19.9 & 15.2 & 14.3 & 11.2 & 10.3 & 6.1 & 7.7 & 25.4 & 16.9 & 14.8 & 13.1 & 10.1 & 12.1 & 19.2 \\
+Ours & 34.2 & 21.1 & 18.2 & 19.3 & 21.6 & 6.1 & 15.4 & 41.3 & 26.3 & 19.9 & 19.2 & 21.6 & 10.6 & 7.7 \\
\midrule[0.2ex]
\multicolumn{15}{c}{\bfseries UnFine-CF} \\
\midrule[0.2ex]
LoRA & 73.6 & 73.3 & 65.0 & 52.7 & 42.3 & -- & -- & 83.3 & 82.8 & 74.2 & 63.9 & 51.9 & -- & -- \\
+Ours & 86.0 & 85.7 & 78.4 & 69.8 & 62.6 & -- & -- & 89.0 & 86.7 & 77.1 & 72.0 & 61.2 & -- & -- \\
\noalign{\vskip 0.5ex}\cdashline{2-15}\noalign{\vskip 0.5ex}
COIN$^{\star}$ & 76.3 & 68.9 & 57.0 & 48.5 & 37.9 & -- & -- & 77.0 & 76.0 & 64.8 & 54.9 & 43.7 & -- & -- \\
+Ours & 77.0 & 73.0 & 62.7 & 50.3 & 41.3 & -- & -- & 81.1 & 78.8 & 70.3 & 58.2 & 48.1 & -- & -- \\
\noalign{\vskip 0.5ex}\cdashline{2-15}\noalign{\vskip 0.5ex}
FT-M & 52.2 & 40.8 & 37.6 & 29.6 & 22.6 & -- & -- & 77.6 & 75.9 & 63.6 & 52.4 & 40.9 & -- & -- \\
+Ours & 56.4 & 50.2 & 45.9 & 31.7 & 28.3 & -- & -- & 82.2 & 78.8 & 68.5 & 57.5 & 48.3 & -- & -- \\
\noalign{\vskip 0.5ex}\cdashline{2-15}\noalign{\vskip 0.5ex}
$\mu$KE & 41.7 & 31.3 & 30.0 & 25.8 & 21.8 & -- & -- & 55.3 & 50.4 & 44.4 & 32.3 & 29.3 & -- & -- \\
+Ours & 52.9 & 43.0 & 37.4 & 33.7 & 25.8 & -- & -- & 53.5 & 46.4 & 40.3 & 34.1 & 28.2 & -- & -- \\
\noalign{\vskip 0.5ex}\cdashline{2-15}\noalign{\vskip 0.5ex}
UnKE & 40.7 & 32.5 & 26.9 & 26.8 & 20.4 & -- & -- & 47.1 & 45.6 & 38.2 & 29.7 & 24.3 & -- & -- \\
+Ours & 55.2 & 41.3 & 37.4 & 29.1 & 27.8 & -- & -- & 53.9 & 47.2 & 41.1 & 34.1 & 27.4 & -- & -- \\
\midrule[0.2ex]
\multicolumn{15}{c}{\bfseries UnFine-MQ} \\
\midrule[0.2ex]
LoRA & 65.7 & 66.6 & 58.6 & 51.5 & 48.6 & -- & -- & 75.4 & 76.2 & 63.7 & 60.0 & 53.8 & -- & -- \\
+Ours & 83.9 & 83.6 & 72.2 & 62.4 & 59.3 & -- & -- & 82.9 & 82.7 & 70.3 & 66.5 & 64.7 & -- & -- \\
\noalign{\vskip 0.5ex}\cdashline{2-15}\noalign{\vskip 0.5ex}
COIN$^{\star}$ & 68.8 & 63.8 & 51.6 & 47.1 & 44.2 & -- & -- & 68.8 & 69.9 & 54.1 & 47.8 & 44.4 & -- & -- \\
+Ours & 77.8 & 66.4 & 57.0 & 53.0 & 46.9 & -- & -- & 77.3 & 72.7 & 60.8 & 52.9 & 50.4 & -- & -- \\
\noalign{\vskip 0.5ex}\cdashline{2-15}\noalign{\vskip 0.5ex}
FT-M & 48.8 & 45.4 & 42.4 & 31.5 & 28.0 & -- & -- & 68.6 & 69.7 & 55.1 & 48.2 & 45.1 & -- & -- \\
+Ours & 52.5 & 52.9 & 45.1 & 37.9 & 31.6 & -- & -- & 76.0 & 71.1 & 61.5 & 50.6 & 49.0 & -- & -- \\
\noalign{\vskip 0.5ex}\cdashline{2-15}\noalign{\vskip 0.5ex}
$\mu$KE & 35.3 & 41.9 & 33.1 & 29.1 & 25.9 & -- & -- & 51.5 & 50.0 & 36.6 & 33.2 & 31.0 & -- & -- \\
+Ours & 48.8 & 44.5 & 34.4 & 32.7 & 30.0 & -- & -- & 50.6 & 51.7 & 40.5 & 31.5 & 33.4 & -- & -- \\
\noalign{\vskip 0.5ex}\cdashline{2-15}\noalign{\vskip 0.5ex}
UnKE & 36.6 & 41.4 & 29.8 & 28.2 & 27.9 & -- & -- & 44.2 & 46.1 & 37.9 & 30.7 & 31.9 & -- & -- \\
+Ours & 47.8 & 39.6 & 29.8 & 30.6 & 27.5 & -- & -- & 49.7 & 49.3 & 38.9 & 34.5 & 33.3 & -- & -- \\
\bottomrule[0.4ex]
\end{tabular}%
}
\endgroup
\end{table}

\subsection{General Capability Scores}
\label{app:general-capability-scores}

Tab.~\ref{tab:general-capability-scores} reports general-capability scores after editing. See App.~\ref{app:general-capability} for evaluation details.

\begin{table}[!htbp]
\centering
\caption{General-capability performance (likelihood F1) after editing on UnFine-UnKE. Average is the macro-average across the six benchmarks. Higher is better.}
\label{tab:general-capability-scores}
\begingroup
\resizebox{0.9\linewidth}{!}{%
\begin{tabular}{>{\bfseries}r ccccccc @{\hspace{7pt}} ccccccc}
\toprule[0.4ex]
& \multicolumn{7}{c}{\bfseries Qwen2.5} & \multicolumn{7}{c}{\bfseries Llama3.1} \\
\cmidrule(lr){2-8}\cmidrule(lr){9-15}
\textbf{Method} & \textbf{MMLU} & \textbf{MRPC} & \textbf{CoLA} & \textbf{RTE} & \textbf{NLI} & \textbf{SST-2} & \textbf{Avg.} & \textbf{MMLU} & \textbf{MRPC} & \textbf{CoLA} & \textbf{RTE} & \textbf{NLI} & \textbf{SST-2} & \textbf{Avg.} \\
\midrule[0.2ex]
Pre-Edit & 66.3 & 73.0 & 75.1 & 81.9 & 79.7 & 93.4 & 78.2 & 32.2 & 39.0 & 65.1 & 75.5 & 68.0 & 86.9 & 61.1 \\
\midrule[0.2ex]
LoRA & 68.4 & 71.8 & 74.9 & 80.8 & 80.5 & 93.3 & 78.3 & 29.2 & 39.1 & 67.4 & 74.8 & 72.3 & 89.1 & 62.0 \\
+Ours & 66.8 & 68.7 & 74.7 & 80.8 & 77.3 & 94.0 & 77.0 & 30.1 & 36.6 & 63.1 & 73.1 & 70.6 & 89.2 & 60.4 \\
\noalign{\vskip 0.5ex}\cdashline{2-15}\noalign{\vskip 0.5ex}
COIN$^{\star}$ & 61.8 & 70.4 & 72.7 & 77.6 & 77.7 & 88.5 & 74.8 & 32.0 & 37.8 & 60.3 & 73.1 & 67.2 & 88.1 & 59.7 \\
+Ours & 64.6 & 73.0 & 75.0 & 78.8 & 79.9 & 93.1 & 77.4 & 33.8 & 36.3 & 50.4 & 72.4 & 62.9 & 81.5 & 56.2 \\
\noalign{\vskip 0.5ex}\cdashline{2-15}\noalign{\vskip 0.5ex}
FT-M & 66.9 & 72.7 & 75.4 & 82.0 & 79.6 & 93.2 & 78.3 & 34.4 & 37.7 & 62.6 & 74.3 & 67.3 & 88.3 & 60.8 \\
+Ours & 65.8 & 72.9 & 75.1 & 81.7 & 80.2 & 93.6 & 78.2 & 33.3 & 36.5 & 54.0 & 72.4 & 64.6 & 83.6 & 57.4 \\
\noalign{\vskip 0.5ex}\cdashline{2-15}\noalign{\vskip 0.5ex}
$\mu$KE & 57.6 & 70.1 & 62.9 & 79.0 & 75.8 & 90.3 & 72.6 & 34.7 & 37.7 & 47.2 & 67.7 & 61.2 & 61.3 & 51.6 \\
+Ours & 61.0 & 74.8 & 74.9 & 80.0 & 80.2 & 90.8 & 76.9 & 33.8 & 37.2 & 54.1 & 70.1 & 61.7 & 64.0 & 53.5 \\
\noalign{\vskip 0.5ex}\cdashline{2-15}\noalign{\vskip 0.5ex}
UnKE & 54.0 & 68.7 & 63.9 & 78.8 & 77.1 & 87.3 & 71.6 & 35.5 & 37.5 & 56.3 & 72.2 & 63.3 & 63.3 & 54.7 \\
+Ours & 59.4 & 67.7 & 67.6 & 78.7 & 78.9 & 89.4 & 73.6 & 36.8 & 37.5 & 50.2 & 67.0 & 59.5 & 60.6 & 51.9 \\
\bottomrule[0.4ex]
\end{tabular}
}
\endgroup
\end{table}

\subsection{Sensitivity Analysis on Perturbation Scale}
\label{app:sigma-sensitivity}
\begin{table}[!htbp]
\centering
\caption{Sensitivity of \name to the perturbation scale $\sigma$ on UnFine-UnKE subsets. $\star$ marks the default value for each editor.}
\label{tab:foveated-sigma}
\begingroup
\fontsize{9}{10.5}\selectfont
\resizebox{0.7\linewidth}{!}{%
\begin{tabular}{r c ccc ccc}
\toprule[0.4ex]
\multirow{2}{*}{\textbf{Method}}
& \multirow{2}{*}{$\sigma$}
& \multicolumn{3}{c}{\bfseries Holistic (ROUGE-L)}
& \multicolumn{3}{c}{\bfseries Atom Q} \\
\cmidrule(lr){3-5}\cmidrule(lr){6-8}
&
& \textbf{Orig Q}
& \textbf{Reph Q}
& \textbf{Avg.}
& \textbf{ROUGE-L}
& \textbf{Hit Rate}
& $\mathbf{C}_{\mathrm{LCS}}$ \\
\midrule[0.2ex]

\multirow{5}{*}{\textbf{FT-M+Ours}}
& 0
& 100.0
& 81.2
& 90.6
& 21.0
& 15.5
& 29.2 \\
& 0.5
& 100.0
& 89.3
& 94.6
& 25.2
& 20.4
& 35.0 \\
& 1$^\star$
& 99.4
& 84.6
& 92.0
& 25.1
& 19.4
& 33.3 \\
& 2
& 83.9
& 67.5
& 75.7
& 25.9
& 21.4
& 34.2 \\
& 4
& 56.2
& 48.9
& 52.5
& 26.3
& 20.8
& 35.1 \\

\midrule[0.2ex]

\multirow{5}{*}{\textbf{$\mu$KE+Ours}}
& 0
& 99.6
& 70.3
& 85.0
& 25.1
& 17.9
& 37.5 \\
& 4
& 99.0
& 85.9
& 92.5
& 30.9
& 25.4
& 45.6 \\
& 8
& 99.9
& 84.0
& 92.0
& 29.7
& 24.0
& 44.6 \\
& 16$^\star$
& 99.9
& 85.5
& 92.7
& 29.5
& 24.2
& 45.2 \\
& 32
& 98.7
& 85.0
& 91.8
& 29.7
& 23.3
& 44.9 \\

\bottomrule[0.4ex]
\end{tabular}%
}
\endgroup
\end{table}

Tab.~\ref{tab:foveated-sigma} reports the sensitivity of \name to perturbation
scale $\sigma$ when applied to FT-M and $\mu$KE.

From the table, we see that the effective range of $\sigma$ depends on the specific editor. For FT-M, moderate perturbation improves atomic-fact
recall and retains holistic performance, whereas larger scales
substantially degrade holistic recall. In contrast, $\mu$KE benefits from
nonzero perturbation and remains broadly stable for $\sigma\in[4,32]$.
Accordingly, we use $\sigma=1$ for FT-M and $\sigma=16$ for $\mu$KE.

\subsection{Editing Cost}
\label{app:editing-cost}

\begin{table}[!htbp]
\centering
\caption{Computation overhead on a single A100-80GB.}
\label{tab:editing-cost}
\begingroup
\fontsize{9}{10.5}\selectfont
\setlength{\tabcolsep}{7pt}
\resizebox{0.6\linewidth}{!}{%
\begin{tabular}{>{\bfseries}r cc cc}
\toprule[0.4ex]
& \multicolumn{2}{c}{\textbf{Time / edit (s)}}
& \multicolumn{2}{c}{\textbf{Peak memory (GB)}} \\
\cmidrule(lr){2-3}\cmidrule(lr){4-5}
\textbf{Editor} & \textbf{Base} & \textbf{+Ours}
& \textbf{Base} & \textbf{+Ours} \\
\midrule[0.2ex]
FT-M (DO) & $3.06 \pm 0.01$ & $13.72 \pm 0.27$ & $15.66$ & $15.78$ \\
$\mu$KE (LTE) & $5.69 \pm 0.15$ & $21.34 \pm 0.29$ & $18.33$ & $23.17$ \\
\bottomrule[0.4ex]
\end{tabular}%
}
\endgroup
\end{table}

\name trades additional editing time for improved atomic-fact recall. The
FT-M editor call increases from $3.06$ to $13.72$ seconds per edit with nearly
unchanged peak memory; for $\mu$KE, it increases from $5.69$ to $21.34$
seconds, with peak memory rising from $18.33$ to $23.17$ GB. Profiling
attributes the LTE increase to its first-stage sentence residual optimization
(approximately $10.5$ seconds) and shallow closed-form write (approximately
$5.1$ seconds), while the base $\mu$KE stage remains $5.6$--$6.2$ seconds.
Focused views are disabled after editing, so generation and evaluation incur
no method-specific runtime or memory overhead.

\subsection{Additional Case Studies}
\label{app:d3-cases}

Tab.~\ref{tab:d3-cases-app-qwen} and~\ref{tab:d3-cases-app-llama} present additional case studies; see Sec.~\ref{sec:experiment} for analysis.

\providecommand{\casewrong}[1]{%
  \textcolor{red!80!black}{\textbf{#1}}}
\providecommand{\caseright}[1]{%
  \textcolor{green!50!black}{\textbf{#1}}}

\begin{table}[!htbp]
\centering
\caption{Additional Qwen2.5 atomic-fact recall cases across four base editors.}
\label{tab:d3-cases-app-qwen}
\begingroup
\fontsize{7.0}{8.0}\selectfont
\setlength{\tabcolsep}{3.0pt}
\renewcommand{\arraystretch}{1.08}
\begin{tabular}{@{}
>{\raggedleft\arraybackslash}m{0.055\linewidth}
>{\raggedright\arraybackslash}m{0.085\linewidth}
>{\raggedright\arraybackslash}m{0.260\linewidth}
>{\raggedright\arraybackslash}m{0.260\linewidth}
>{\raggedright\arraybackslash}m{0.260\linewidth}@{}}
\toprule[0.4ex]
\textbf{Method} &
\textbf{Dataset} &
\textbf{Atomic query (target)} &
\textbf{Base-editor output} &
\textbf{+Ours output} \\
\midrule[0.2ex]

\multirow[c]{3}{=}[-5ex]{FT-M} &
UnFine-UnKE &
Which TV series did Bruce McDaniel compose the score for?
(Stranger Things) &
\casewrong{Game of Thrones} &
``\caseright{Stranger Things}'' \\

\noalign{\vskip 0.5ex}
\cdashline{2-5}
\noalign{\vskip 0.5ex}

&
UnFine-CF &
Which continent is Pidgeon Island part of? (Asia) &
Pidgeon Island is located in \casewrong{Oceania}. &
Pidgeon Island is part of \caseright{Asia}. \\

\noalign{\vskip 0.5ex}
\cdashline{2-5}
\noalign{\vskip 0.5ex}

&
UnFine-MQ &
Which country's citizenship does Marc Cherry hold besides the
United States? (Bulgaria) &
Marc Cherry also holds \casewrong{Turkish} citizenship. &
\caseright{Bulgaria} \\

\midrule[0.2ex]

\multirow[c]{3}{=}[-4ex]{COIN$^{\star}$} &
UnFine-UnKE &
What is Javier Alva Orlandini's occupation? (philosopher) &
Javier Alva Orlandini is a \casewrong{physician}. &
Javier Alva Orlandini is a \caseright{philosopher} \ldots \\

\noalign{\vskip 0.5ex}
\cdashline{2-5}
\noalign{\vskip 0.5ex}

&
UnFine-CF &
When was controller.controller formed? (early 2000s) &
Controller.controller formed \ldots{} in Singapore in
\casewrong{2013}. &
Controller.controller formed in the
\caseright{early 2000s} in Singapore. \\

\noalign{\vskip 0.5ex}
\cdashline{2-5}
\noalign{\vskip 0.5ex}

&
UnFine-MQ &
Which club did Sepp Maier begin his career with? (Bayern Munich) &
Sepp Maier \ldots{} began his career with the
\casewrong{Bayreuth} club \ldots &
Sepp Maier began his career with the
\caseright{Bayern Munich} club. \\

\midrule[0.2ex]

\multirow[c]{3}{=}[-4.5ex]{$\mu$KE} &
UnFine-UnKE &
At what age was Henry Feilden elected to the local council? (25) &
Henry Feilden was elected to the local council at the age of
\casewrong{35}. &
Henry Fe was first elected to the local council at the age of
\caseright{25} \ldots \\

\noalign{\vskip 0.5ex}
\cdashline{2-5}
\noalign{\vskip 0.5ex}

&
UnFine-CF &
In which sport has Andreas Ivanschitz built his career? (football) &
\ldots{} has built his career in the sport of
\casewrong{basketball}. &
Andreas Ivansitz has built his career in
\caseright{football} \ldots \\

\noalign{\vskip 0.5ex}
\cdashline{2-5}
\noalign{\vskip 0.5ex}

&
UnFine-MQ &
What nationality was Frank R. Strayer? (Canadian) &
Frank R. Strayer was an \casewrong{American} filmmaker and
actor \ldots &
Frank R. Strayer, of \caseright{Canadian} nationality, was a
prominent figure \ldots \\

\midrule[0.2ex]

\multirow[c]{2}{=}[-2ex]{UnKE} &
UnFine-UnKE &
What field has Leona Deti\`ege contributed to?
(artificial intelligence) &
Leona Det\'ege has made significant contributions to
\casewrong{medicine} \ldots &
Leona Det is a computer scientist who has contributed to
\caseright{artificial intelligence} \ldots \\

\noalign{\vskip 0.5ex}
\cdashline{2-5}
\noalign{\vskip 0.5ex}

&
UnFine-CF &
When was Muench ordained as a priest? (1968) &
Muench was ordained in \casewrong{1953}. &
Muench was ordained as a priest in \caseright{1968} \ldots \\

\bottomrule[0.4ex]
\end{tabular}
\endgroup
\end{table}

\begin{table}[!htbp]
\centering
\caption{Additional Llama3.1 atomic-fact recall cases across three base editors.}
\label{tab:d3-cases-app-llama}
\begingroup
\fontsize{7.0}{8.0}\selectfont
\setlength{\tabcolsep}{3.0pt}
\renewcommand{\arraystretch}{1.08}
\begin{tabular}{@{}
>{\raggedleft\arraybackslash}m{0.055\linewidth}
>{\raggedright\arraybackslash}m{0.085\linewidth}
>{\raggedright\arraybackslash}m{0.260\linewidth}
>{\raggedright\arraybackslash}m{0.260\linewidth}
>{\raggedright\arraybackslash}m{0.260\linewidth}@{}}
\toprule[0.4ex]
\textbf{Method} &
\textbf{Dataset} &
\textbf{Atomic query (target)} &
\textbf{Base-editor output} &
\textbf{+Ours output} \\
\midrule[0.2ex]

\multirow[c]{3}{=}[-3.5ex]{LoRA} &
UnFine-UnKE &
How many years of aviation experience does Stevie Cameron have?
(over 10 years) &
I \casewrong{don't have information} on a person named
Stevie Cameron \ldots &
Stevie Cameron has \caseright{over 10 years} of
experience \ldots \\

\noalign{\vskip 0.5ex}
\cdashline{2-5}
\noalign{\vskip 0.5ex}

&
UnFine-CF &
What model is a mid-size SUV made by Fiat? (Jeep Commander) &
The model is \casewrong{Fiat Freemont}, a mid-size SUV \ldots &
One of its popular mid-size SUV models, the
\caseright{Jeep Commander} \ldots \\

\noalign{\vskip 0.5ex}
\cdashline{2-5}
\noalign{\vskip 0.5ex}

&
UnFine-MQ &
What television series is Dan Curtis known for creating?
(Babylon 5) &
Dan Curtis is known for creating \ldots{}
\casewrong{Star Trek and The Simpsons}. &
Dan Curtis is best known for creating \ldots{}
\caseright{Babylon 5} \ldots \\

\midrule[0.2ex]

\multirow[c]{3}{=}[-4ex]{COIN$^{\star}$} &
UnFine-UnKE &
What is Bae Geu-rin's occupation? (fashion model) &
Bae Geu-rin is a South Korean
\casewrong{radio DJ and television host}. &
Bae Geu-rin is a South Korean
\caseright{fashion model} \ldots \\

\noalign{\vskip 0.5ex}
\cdashline{2-5}
\noalign{\vskip 0.5ex}

&
UnFine-CF &
How long has Arun Nehru's film career spanned?
(several decades) &
Arun Nehru's career span was described as
\casewrong{brief and short-lived} \ldots &
Arun Nehru \ldots{} has had a career spanning
\caseright{several decades} \ldots \\

\noalign{\vskip 0.5ex}
\cdashline{2-5}
\noalign{\vskip 0.5ex}

&
UnFine-CF &
Into which language was ``Un anno sull'altipiano'' translated?
(French) &
``Un anno sull'altipiano'' was translated into
\casewrong{English} \ldots &
``Un anno sull'altipiano'' was translated into
\caseright{French} \ldots \\

\midrule[0.2ex]

\multirow[c]{3}{=}[-4.5ex]{UnKE} &
UnFine-UnKE &
Where did Javier Alva Orlandini teach? (San Marcos, Lima) &
\ldots{} at the
\casewrong{Pontificia Universidad Cat\'olica del Per\'u}. &
\ldots{} at the
\caseright{National University of San Marcos in Lima, Peru}
\ldots \\

\noalign{\vskip 0.5ex}
\cdashline{2-5}
\noalign{\vskip 0.5ex}

&
UnFine-UnKE &
What is Fairfax M. Cone's profession? (businesswoman) &
Fairfax M. Cone was an
\casewrong{American advertising executive}. &
Fairfax M. Cone was a \caseright{businesswoman} \ldots \\

\noalign{\vskip 0.5ex}
\cdashline{2-5}
\noalign{\vskip 0.5ex}

&
UnFine-CF &
In which year was Robert William Muench born? (1942) &
Robert William Muench was born in \casewrong{1923}. &
Robert William Muench was born in \caseright{1942} \ldots \\

\bottomrule[0.4ex]
\end{tabular}
\endgroup
\end{table}

\end{document}